\documentclass{article}
\usepackage{amsmath,amsfonts,amssymb,amsthm,bm,upgreek}
\usepackage{fullpage, parskip}
\usepackage[square,numbers]{natbib}
\usepackage[utf8]{inputenc} 
\usepackage[T1]{fontenc}    

\usepackage[colorlinks=true,allcolors=blue]{hyperref}     
\usepackage[dvipsnames]{xcolor}

\usepackage{url}            
\usepackage{booktabs}   
\usepackage{xfrac}
\usepackage{amsfonts}       
\usepackage{nicefrac}       
\usepackage{microtype}      
\usepackage{macros_arxiv}

\newif\ifcoltshort
\coltshortfalse

\title{Efficient Convex Optimization Requires Superlinear Memory
}
\author{%
   Annie Marsden \\
   Stanford University \\
  \texttt{marsden@stanford.edu}
  \and
   Vatsal Sharan \\
   USC\thanks{Part of the work was done while the author was at MIT.} \\
  \texttt{vsharan@usc.edu}
  \and
   Aaron Sidford \\
   Stanford University \\
  \texttt{sidford@stanford.edu}
   \and
   Gregory Valiant \\
   Stanford University \\
  \texttt{valiant@stanford.edu}
}
\begin{document}
\date{}
\maketitle
\thispagestyle{empty}
\begin{abstract}
 We show that any memory-constrained, first-order algorithm which minimizes $d$-dimensional, $1$-Lipschitz convex functions over the unit ball to $1/\poly(d)$ accuracy using at most $d^{1.25 - \delta}$ bits of memory must make at least $\tilde{\Omega}(d^{1 + (4/3)\delta})$ first-order queries (for any constant $\delta \in [0, 1/4]$). Consequently, the performance of such memory-constrained algorithms are at least a polynomial factor worse than the optimal $\tilde{O}(d)$ query bound for this problem obtained by cutting plane methods that use $\tilde{O}(d^2)$ memory. This resolves one of the open problems in the COLT 2019 open problem publication of Woodworth and Srebro.
\end{abstract}

\section{Introduction}

Minimizing a convex objective function $f: \R^d \to \R$ given access to a first-order oracle---that returns the function evaluation and (sub)gradient $(f(\x),\nabla f(\x))$ when queried for point $\x$---is a canonical problem and fundamental primitive in machine learning and optimization.   
There are methods that, given any $1$-Lipschitz, convex $f : \R^d \rightarrow \R$ accessible via a first-order oracle, compute an $\epsilon$-approximate minimizer over the unit ball with just $\bigo(\min\{\epsilon^{-2}, d \log(1/\epsilon)\})$ queries to the oracle. This query complexity is worst-case optimal~\citep{NemirovskiYu83} and foundational in optimization theory. $\bigo(\epsilon^{-2})$ queries is achievable using subgradient descent; this is a simple, widely-used, eminently practical 
 method that solves the problem using a total of $\bigo(d \epsilon^{-2})$ arithmetic operations on $\bigo(\log(d/\epsilon))$-bit numbers.
 On the other hand, building on the $\bigo(d^2 \log(1/\epsilon))$ query complexity of the well-known ellipsoid method \citep{yudin1976informational,shor1977cut}, different cutting plane methods achieve a query complexity of $\bigo(d \log(1/\epsilon))$, e.g. center of mass with sampling based techniques \citep{levin1965algorithm,bertsimas2004solving},  
 volumetric center \citep{vaidya1989new,atkinson1995cutting}, inscribed ellipsoid \citep{khachiyan1988method,nesterov1989self}; these methods are perhaps less frequently used in practice and large-scale learning and all use at least $\Omega(d^3 \log(1/\epsilon))$-time, even with recent improvements \citep{LeSiWong15, JiLeSoWong20}.\footnote{This arises from $\Omega(d \log(1/\epsilon))$ iterations of working in different change of basis or solving a linear system, each of which takes $\Omega(d^2)$-time naively.
 }

Though state-of-the-art cutting plane methods have larger computational overhead and are sometimes regarded as impractical, for small enough $\epsilon$, they give the state-of-the-art query bounds. Further, in different theoretical settings, e.g.,  
semidefinite programming \citep{anstreicher2000volumetric,kartik12sdp,LeSiWong15}, submodular optimization \citep{mccormick2005submodular,grotschel2012geometric,LeSiWong15,jiang2021minimizing} and equilibrium computation \citep{papadimitriou2008computing,jiang2011polynomial}, cutting-plane-methods have yielded state-of-the-art runtimes at various points of time. This leads to the natural question of what is needed of a method to significantly outperform gradient descent and take advantage of the improved query complexity enjoyed by cutting plane methods? Can we design methods that obtain optimal query complexities while maintaining the practicality 
of gradient descent methods?

The COLT 2019 open problem ~\cite{woodworth2019open} proposed tackling this question through the lens of memory: to what extent is large memory necessary to achieve improved query complexity?  {Any algorithm for the problem must use $\Omega\left(d \log(1/\eps\right)$ memory irrespective of its query complexity, since that much memory is needed to represent the answer \citep{woodworth2019open}.
Subgradient descent matches this lower bound and can be implemented with only $\bigo(d \log(1/\epsilon))$-bits of memory \citep{woodworth2019open}. However,}  all known methods that achieve a query complexity significantly better than gradient descent, e.g. cutting plane methods, use $\Omega(d^2 \log(1/\epsilon))$ bits of memory. Understanding the trade-offs in memory and query complexity could inform the design of future efficient optimization methods.

In this paper we show that memory \emph{does} play a critical role in  attaining an optimal query complexity for convex optimization. Our main result is the following theorem which shows that any algorithm whose memory usage is sufficiently small (though potentially, still superlinear) must make polynomially more queries to a first-order oracle than cutting plane methods. Specifically, any algorithm that uses significantly less than $d^{1.25}$ bits of memory requires a polynomial factor more first-order queries than the optimal $\bigo(d \log(d))$ queries achieved by quadratic memory cutting plane methods. {This establishes the first non-trivial tradeoff between memory and query complexity for convex optimization with access to a first-order oracle.}

\begin{figure}
    \centering
    \includegraphics[width=0.6\linewidth]{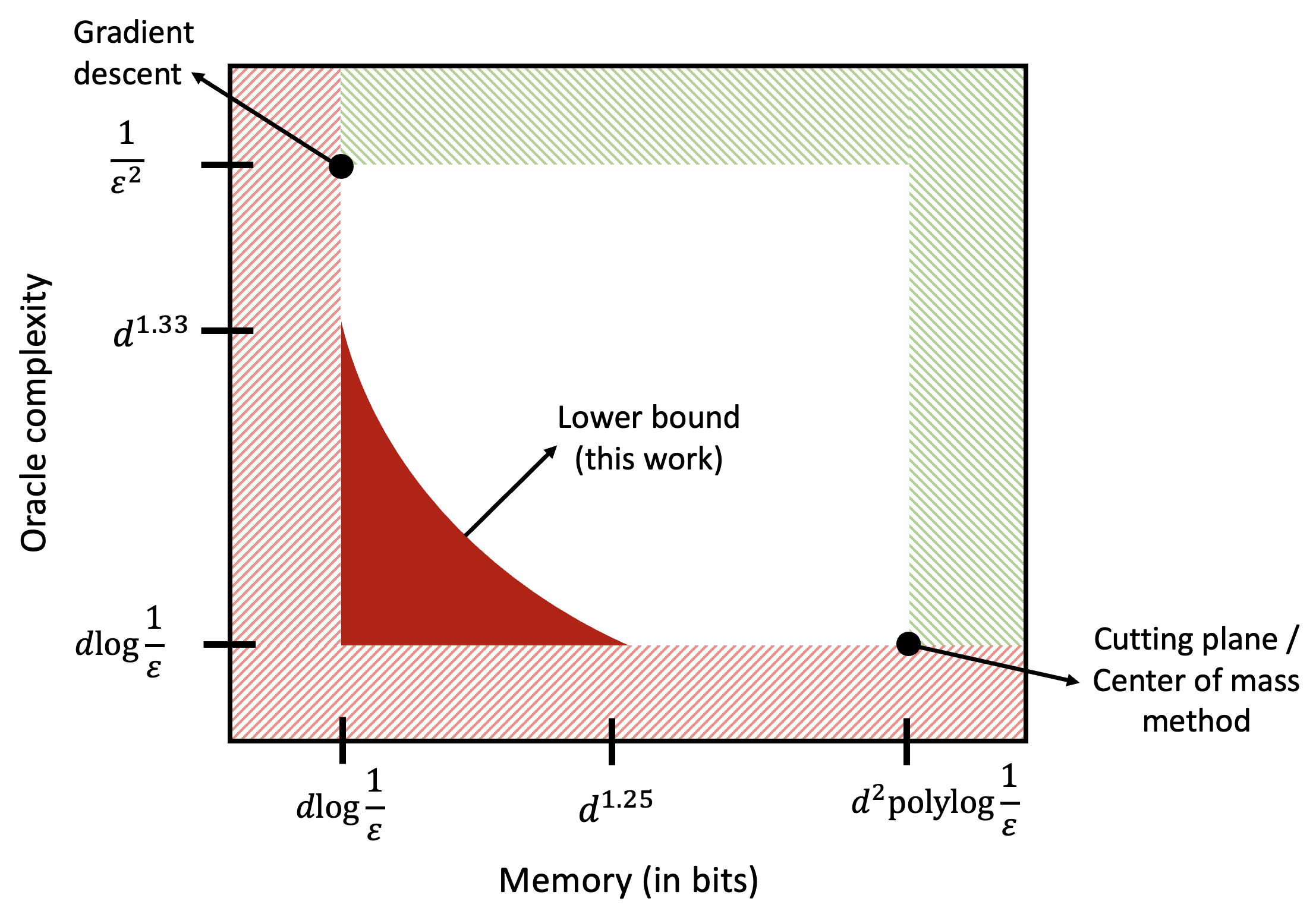}
    \caption{Tradeoffs between available memory and first-order oracle complexity for minimizing 1-Lipschitz convex functions over the unit ball (adapted from \cite{woodworth2019open}). The dashed red 
    region corresponds to information-theoretic lower bounds on the memory and query-complexity. The dashed green 
    region corresponds to known upper bounds. This work shows that the solid red region is not achievable for any algorithm.
    }
    \label{fig:tradeoff}
\end{figure}

\begin{theorem}
\label{thm:clean}
 For some {$\epsilon \ge 1/d^7$} and any $\delta \in [0,\sfrac{1}{4}]$, any (potentially randomized) algorithm which outputs an $\epsilon$-optimal point with  probability at least $\sfrac{2}{3}$ given first-order oracle access to any $1$-Lipschitz convex function must use either at least $d^{1.25 - \delta}$ bits of memory or make $\tilde{\Omega}(d^{1+ \frac{4}{3}\delta})$ first-order queries (where the $\tilde{\Omega}$ notation hides factors poly-logarithmic in $d$).
\end{theorem}
Beyond shedding light on the complexity of a fundamental memory-constrained optimization problem, we provide several tools for establishing such lower bounds. In particular, we introduce a set of properties which are sufficient for an optimization problem to exhibit a memory lower bound and provide an information-theoretic framework to prove these lower bounds. We hope these tools are an aid to future work on the role of memory in optimization.

This work fits within the broader context of understanding fundamental resource tradeoffs for optimization and learning. For many settings, establishing (unconditional) query/time or memory/time tradeoffs is notoriously hard---perhaps akin to P vs NP (e.g. providing time lower bounds for cutting plane methods).  Questions of memory/query and memory/data tradeoffs, however, have a more information theoretic nature and hence seem more approachable.  Together with the increasing importance of memory considerations in large-scale optimization and learning, there is a strong case for pinning down the landscape of such tradeoffs, which may offer a new perspective on the current suite of algorithms and inform the effort to develop new ones.

\subsection{Technical Overview and Contributions}\label{sec:tech_overview}
To prove Theorem~\ref{thm:clean}, we provide an explicit distribution over functions that is hard for any memory-constrained randomized algorithm to optimize. Though the proof requires care and we introduce a variety of machinery to obtain it, this lower bounding family of functions is simple to state. The function is a variant of the so-called ``Nemirovski function,'' which  has been used to show lower bounds for highly parallel non-smooth convex optimization \citep{Nemirovski94, BubeckJiLeLiSi19, BalkanskiSi18}.

Formally, our difficult class of functions for memory size $M$ is constructed as follows: for some $\gamma >0$ and $\nnem = \tildeo \left(( d^2/\memsize )^{1/3}\right)$  let $\v_1,\ldots,\v_N$ be unit vectors drawn i.i.d. from the $d$ dimensional scaled hypercube $\v_i \simiid \mathsf{Unif} \left( d^{-1/2}  \mathcal{H}_d \right)$ and let $\vec{a}_1,\ldots,\vec{a}_{\lfloor d/2 \rfloor}$ be drawn i.i.d. from the hypercube, $\vec{a}_j \sim \mathsf{Unif} \left( \mathcal{H}_d\right)$ where $\alpha \mathcal{H}_d \defeq \left \{ \pm \alpha \right \}^d$. Let $\mat{A} = \left( \vec{a}_1, \dots, \vec{a}_{\lfloor d/2 \rfloor} \right)$ and 
\begin{equation}
    F(\x) = (1/d^6) \max \left \{  d^5 \linf{\A \x} - 1, \max_{i \in [\nnem]} \v_i^{\top} \x  - i \gamma \right \}.\label{eq:hard_function}
\end{equation}

We describe the key components of this function and the intuition behind their role in Section~\ref{sec:proof_strategy}.  At a high level, we  show that any algorithm that optimizes this function must, with high probability, observe each of the $\v_1,\ldots,\v_N$ and will  observe $\v_i$ before $\v_{i+1}$ due to the $i \gamma$ term.  The intuition for the memory-query tradeoff is that after seeing $\v_{i}$, in order to see $\v_{i+1}$, the algorithm must query a vector near the nullspace of $\mat{A}$ but with a significant component in the direction of $\v_{i}$. Since  $\v_{i}$ is random, this requires finding a ``new'' vector in the nullspace of $\mat{A}$. This implies that for some appropriately chosen $k$ and a block of $k$ Nemirovki vectors $\{\v_i,\v_{i+1},\ldots,\v_{i+k-1}\}$, to see $\{\v_{i+1},\v_{i+2},\ldots,\v_{i+k}\}$ after seeing $\v_{i}$ the algorithm must query $k$ sufficiently independent vectors in the null space of $\mat{A}$. We show that an algorithm with insufficient memory to store $\mat{A}$ cannot find such a large set of sufficiently independent vectors in the nullspace of $\mat{A}$, and must essentially ``re-learn'' $\mat{A}$ for every such block of $k$ Nemirovski vectors.

Rather than give a direct proof of \cref{thm:clean} using the explicit function in Eq. \ref{eq:hard_function} we provide a more abstract framework which gives broader insight into which kinds of functions could lead to non-trivial memory-constrained lower bounds, and which might lead to tighter lower bounds in the future. To that end we introduce the notion of a \emph{\functionclassname} which delineates the key properties of a distribution over functions that lead to memory-constrained lower bounds.  We show that for such functions, the problem of memory-constrained optimization is at least as hard as the following problem of finding a set of vectors which are approximately orthogonal to another set of vectors. The problem is motivated by the proof intuition discussed above, where we argued that minimizing Eq. \ref{eq:hard_function} requires finding a large number of sufficiently independent vectors in the null space of $\mat{A}$.

\begin{definition}[Informal version of the \orthgame]\label{def:informal_game}
Given $\A \in \{\pm 1\}^{d/2 \times d}$, \playerone's objective is to return a set of $\nreturn$ vectors $\{\y_1,\dots, \y_{\roundlen}\}$ which satisfy
\begin{enumerate}
    \item 
    $\forall  i \in [\nreturn], \y_i$ is approximately orthogonal to all the rows of $\A: \linf{\mat{A} \y_i}/\norm{\y_i}_2 \leq \thetathreshold$.
    \item The set of vectors $\{\y_1,\dots, \y_{\roundlen}\}$ is \emph{robustly linearly independent}:  denoting $S_0=\emptyset, S_i=\spcommand(\y_1,\dots, \y_{i})$, $\norm{\proj_{S_{i-1}}(\y_i)}_2/\norm{\y_i}_2 \le 1-1/d^2$,
\end{enumerate}
where the notation $\proj_S(\x)$ denotes the vector in the subspace $S$ which is closest in $\norm{\cdot}_2$ to $\x$.
The game proceeds as follows: \capsplayerone\ first observes $\A$ and stores an $M$-bit $\messages$ about $\A$. She does not subsequently have free access to $\A$, but can adaptively make up to $m$ queries as follows: for $i \in [m]$, based on $\messages$ and all previous queries and their results, she can request any row $i \in [d/2]$ of the matrix $\A$. Finally, she  outputs a set of $\nreturn$  vectors as a function of  $\messages$ and all $m$ queries and their results.  
\end{definition}
 Note that \playerone\ can trivially win the  game for $M\ge {\Omega}(d\nreturn), m=0$ (by just storing a satisfactory set of $\nreturn$ vectors in the $\messages$) and for $M=0, m=d/2$ (by querying all rows of $\A$). We show a lower bound that this is essentially all that is possible: \emph{for $\A$ sampled uniformly at random from $\pmset^{d/2 \times d}$, if $M$ is a constant factor smaller than $d\roundlen$, then \playerone\ must make at least $d/5$ queries to win with probability at least $2/3$}.  Our analysis proceeds via an intuitive information-theoretic framework, which could have applications for showing query lower bounds for memory-constrained algorithms in other optimization and learning settings. We sketch the analysis in Section \ref{sec:entropy_sketch}.
\subsection{Related Work}

\paragraph{Memory-sample tradeoffs for learning}

There is a recent line of work to understand learning under information constraints such as limited memory or communication constraints \citep{balcan2012distributed,duchi2013local,ZhangDuJoWa13,garg2014communication, shamir2014fundamental, arjevani2015communication, SteinhardtDu15, SteinhardtVaWa16, BravermanGaMaNgWo16, DaganSh18, dagan2019space, woodworth2021min}. Most of these results obtain lower bounds for the regime when the available memory is less than that required to store a single datapoint (with the notable exception of \cite{DaganSh18} and \cite{dagan2019space}). However the breakthrough paper \cite{Raz17} showed an exponential lower bound on the number of random examples needed for learning parities with memory as large as quadratic. Subsequent work extended and refined this result to multiple learning problems over finite fields \citep{MoshkovitzMo17,BeameGhOvXi18,MoshkovitzMo18,KolRaTa17,Raz18,GargRaTa18}.

Most related to our line of work is \cite{sharan2019memory}, which considers the continuous valued learning/optimization problem of performing linear regression given access to randomly drawn examples from an isotropic Gaussian. They show that any sub-quadratic memory algorithm for the problem needs $\Omega(d\log\log(1/\eps)))$ samples to find an $\eps$-optimal solution for $\eps \le 1/d^{\Omega(\log d)}$, whereas in this regime an algorithm with memory $\Tilde{\bigo}(d^2)$ can find an $\eps$-optimal solution with only $d$ examples. Since each example provides an unbiased estimate of the expected regression loss, this translates to a lower bound for convex optimization given access to a stochastic gradient oracle. However the upper bound of $d$ examples is not a generic convex optimization algorithm/convergence rate but comes from the fact that the linear systems can be solved to the required accuracy using $d$ examples.

There is also significant work on memory lower bounds for streaming algorithms, e.g. \citep{alon1999space,bar2004information,clarkson2009numerical,dagan2019space}, where the setup is that the algorithm only gets  a single-pass over a data stream.
The work of \cite{dagan2019space} mentioned earlier uses an Approximate Null-Vector Problem (ANVP) to show memory lower bounds, which shares some similarities with the Orthogonal Vector Game used in our proof (informally presented in Definition~\ref{def:informal_game}). In the ANVP, the objective is to find a single vector approximately orthogonal to a stream of vectors. The goal of the ANVP is similar to the goal of the Orthogonal Vector Game, which is to find a set of such vectors. However, the key difference is that the ANVP is in the (one-pass) streaming setting, whereas the Orthogonal Vector Game allows for stronger access to the input in two senses. First, the Orthogonal Vector Game algorithms see the entire input and can store a $M$-bit $\messages$ about the input, and second, the algorithms adaptively query the input. The main challenge in our analysis (discussed later in Section~\ref{sec:entropy_sketch}) is to bound the power of these adaptively issued queries in conjunction with the $\messages$.
\paragraph{Lower bounds for convex optimization}
Starting with the early work of \cite{NemirovskiYu83}, there is extensive literature on lower bounds for convex optimization. Some of the key results in this area include classic lower bounds for finding approximate minimizers of Lipschitz functions with access to a subgradient oracle \citep{NemirovskiYu83,nesterov2003introductory,braun2017lower}, including recent progress on lower bounds for randomized algorithms \citep{woodworth2016tight,woodworth2017lower,simchowitz2018tight,simchowitz2018randomized,braverman2020gradient,sun2021querying}. There is also work on the effect of parallelism on these lower bounds \citep{Nemirovski94,BalkanskiSi18,diakonikolas2019lower,BubeckJiLeLiSi19}. For a broader introduction to the oracle complexity of optimization, we refer the reader to surveys such as \cite{nesterov2003introductory} and \cite{Bubeck14}.
\paragraph{Memory-limited optimization algorithms}
While the focus of this work is lower bounds, there is a long line of work on developing memory-efficient  optimization algorithms, including methods that leverage second-order structure via first-order methods. Such methods include Limited-memory-BFGS \citep{Nocedal-80,Liu.Nocedal-MP89} and the conjugate gradient (CG) method for solving linear systems \citep{Hestenes.Stiefel-52} and various non-linear extensions of CG \citep{Fletcher.Reeves-64,Hager.Zhang-06} and methods based on subsampling and sketching the Hessian \citep{pilanci2017newton,xu2020newton,roosta2019sub}. A related line of work is on communication-efficient optimization algorithms for distributed settings (see \cite{jaggi2014communication,shamir2014communication,jordan2018communication,alistarh2018convergence,ye2018communication} and references therein).

\subsection{Subsequent and Future Work}

Our work is a first step towards establishing optimal memory/query trade-offs for optimization. An immediate open problem suggested by our work is that of improving our bounds, both to match the quadratic memory upper bound of cutting plane methods, establish tight bounds for all values of the desired accuracy, $\epsilon$, and fully resolve the open problems of Woodworth and Srebro \cite{woodworth2019open}. Since the initial publication of this work, there have been several advances in this direction. 

First, recent work of  Blanchard, Zhang and Jaillet \cite{blanchard2023quadratic} built upon our proof framework to show that quadratic memory is indeed necessary for deterministic algorithms to obtain near-optimal query complexities. In our construction in Eq. \ref{eq:hard_function}, we  have $\nnem = \tildeo \left( \left( d^2/\memsize\right)^{1/3} \right)$  Nemirovski vectors $\v_1,\ldots,\v_N$, however, packing more Nemirovski vectors would lead to a stronger lower bound. We discuss this further in Remark~\ref{rem:whyd43}. In the case of deterministic algorithms, \cite{blanchard2023quadratic} provide a clever adaptive strategy which packs $N=\Omega(d)$ Nemirovski vectors and obtains a quadratic memory lower bound.

A second recent paper, Chen and  Peng~\cite{chen2023memory}, established that quadratic memory is necessary to achieve the optimal convergence rate even for randomized algorithms, at least in the parameter regime where $\epsilon$ is quasipolynomially small in $d$.  Their proof leverages the same hard distribution of functions as we describe (Eq. \ref{eq:hard_function}), but introduces a variant of our Orthogonal Vector Game and develops an elegant analysis for it which allows them to show that a significant number of queries must be expended to observe \emph{each} new Nemirovski vector $v_i$.  In contrast, our proof blocks together groups of such vectors.

Another recent work of Blanchard \cite{blanchard2024gradient} considers the problem of finding a feasible point in a given set with access to a separation oracle. They construct an oracle which requires the algorithm to find robustly linearly independent vectors which are orthogonal to a sequence of subspaces. They build upon our proof framework and develop a novel recursive lower bound technique which allows them to show that gradient descent is almost optimal for the feasibility problem among linear memory algorithms. Their result also implies an exponential lower bound for sub-quadratic memory algorithms to solve the feasibility problem to sufficiently small error.

On the upper bound side, \cite{blanchard2023memory} develops a family of recursive cutting-plane algorithms which use a memory that can vary between $\tilde{O}(d)$ and $\tilde{O}(d^2)$. Previously, algorithms with established query complexities used either $\tilde{O}(d)$ memory or $\tilde{\Omega}(d^2)$.  This work establishes new upper bounds on memory-constrained query complexity for any memory constraint that falls between linear and quadratic. In the regime when $\epsilon\le d^{-\Omega (d)}$, their algorithm matches the memory usage of gradient descent while improving on gradient descent's query complexity.

Even in light of these advances, multiple open problems remain.
For example, relaxing the type of functions and oracles for which we can prove memory-constrained lower bounds is a natural further direction. Interestingly, our proofs (and those of ~\cite{blanchard2023quadratic,chen2023memory}) rely on the assumption that we are optimizing a Lipschitz but non-smooth function, and that at a point queried we only observe the function's value and a single subgradient (rather than the set of all subgradients or all information about the function in a neighborhood) and consequently standard smoothing techniques do not readily apply \citep{diakonikolas2019lower,carmon2021thinking}. Proving lower bounds for smooth functions and beyond could illuminate the necessity of larger memory footprints for prominent memory-intensive optimization methods with practical appeal (e.g., interior point methods and quasi-Newton methods). {Additionally, finding algorithms with new tradeoffs between space usage and query complexity would be very interesting. In particular, improving the memory versus query tradeoffs for more structured functions, e.g., piecewise linear functions, is an interesting question for future work.}

\section{Setup and Overview of Results}
\label{section:prelims}
We consider optimization methods for minimizing convex, Lipschitz-continuous functions $F: \R^d \to \R$ over the unit-ball\footnote{By a blackbox-reduction in \cref{appendix:constrainwlog} our results extend to unconstrained optimization while only losing $\poly(d)$ factors in the accuracy $\epsilon$ for which the lower bound applies.} with access to a first-order oracle. Our goal is to understand how the (oracle) query complexity of algorithms is affected by restrictions on the memory available to the algorithm. Our results apply to a rather general definition of \emph{memory-constrained} first-order algorithms given in the following Definition~\ref{def:memoryalgorithm}. Such algorithms include those that use arbitrarily large memory at query time but can save at most $\memsize$ bits of information in between interactions with the first-order oracle. More formally we have:

\begin{definition} [$\memsize$-bit memory-constrained deterministic algorithm]
\label{def:memoryalgorithm}
An \emph{$\memsize$-bit (memory-constrained) deterministic algorithm} with first-order oracle access, $\algdet$, is the iterative execution of a sequence of functions $\left \{ \phi_{\textrm{query},t}, \phi_{\textrm{update},t} \right \}_{t \geq 1}$, where $t$ denotes the iteration number. The function $\phi_{\textrm{query},t}$ maps the $\nth{(t-1)}$ memory state of size at most $\memsize$ bits to the $\nth{t}$ query vector, $\phi_{\textrm{query},t}(\memory_{t-1}) = \x_t$. The algorithm is allowed to use an arbitrarily large amount of memory to execute $\phi_{\textrm{query},t}$. Upon querying $\x_t$, some first-order oracle returns $F(\x_t)$ and a subgradient $\vec{g}_{\x_t} \in \partial F(\x_t)$. The second function maps the first-order information and the old memory state to a new memory state of at most $\memsize$ bits, $\memory_{t} = \phi_{\textrm{update},t}\left(\x_t, F(\x_t), \g_{\x_t}, \memory_{t-1} \right) $. Again, the algorithm may use unlimited memory to execute $\phi_{\textrm{update},t}$.
\end{definition}
Note that our definition of a memory-constrained algorithm is equivalent to the  definition given by \cite{woodworth2019open}. Our analysis also allows for randomized algorithms which will often be denoted as $\algrand$. 
\begin{definition} [$\memsize$-bit memory-constrained randomized algorithm]
\label{def:memoryalgorithmrand}
An \emph{$\memsize$-bit (memory-constrained) randomized algorithm} with first-order oracle access, $\algrand$, is a deterministic algorithm with access to a string $R$ of uniformly random bits which has length $2^d$.\footnote{We remark that $2^d$ can be replaced by any finite-valued function of $d$.}
\end{definition}
In what follows we use the notation $\oracle_F(
\x)$ to denote the first-order oracle which, when queried at some vector $\x$, returns the pair $(F(\x), \g_{\x})$ where $\g_{\x} \in \partial F(\x)$ is a subgradient of $F$ at $\x$. We will also refer to a sub-gradient oracle and use the overloaded notation $\g_F(\x)$ to denote the oracle which simply returns $\g_{\x}$. It will be useful to refer to the sequence of vectors $\x$ queried by some algorithm $\alg$ paired with the subgradient $\g$ returned by the oracle. 
\begin{definition}[Query sequence]
\label{def:querseq}
Given an algorithm $\alg$ with access to some first-order oracle $\oracle$ we let $\seq(\alg, \oracle) = \left \{ (\x_i, \g_i) \right \}_{i}$ denote the \emph{query sequence} of vectors $\x_i$ queried by the algorithm paired with the subgradient $\g_i$ returned by the oracle $\oracle$. 
\end{definition}
\subsection{Proof Strategy}
\label{sec:proof_strategy}

\begin{figure}
    \centering
    \includegraphics[width=0.8\linewidth]{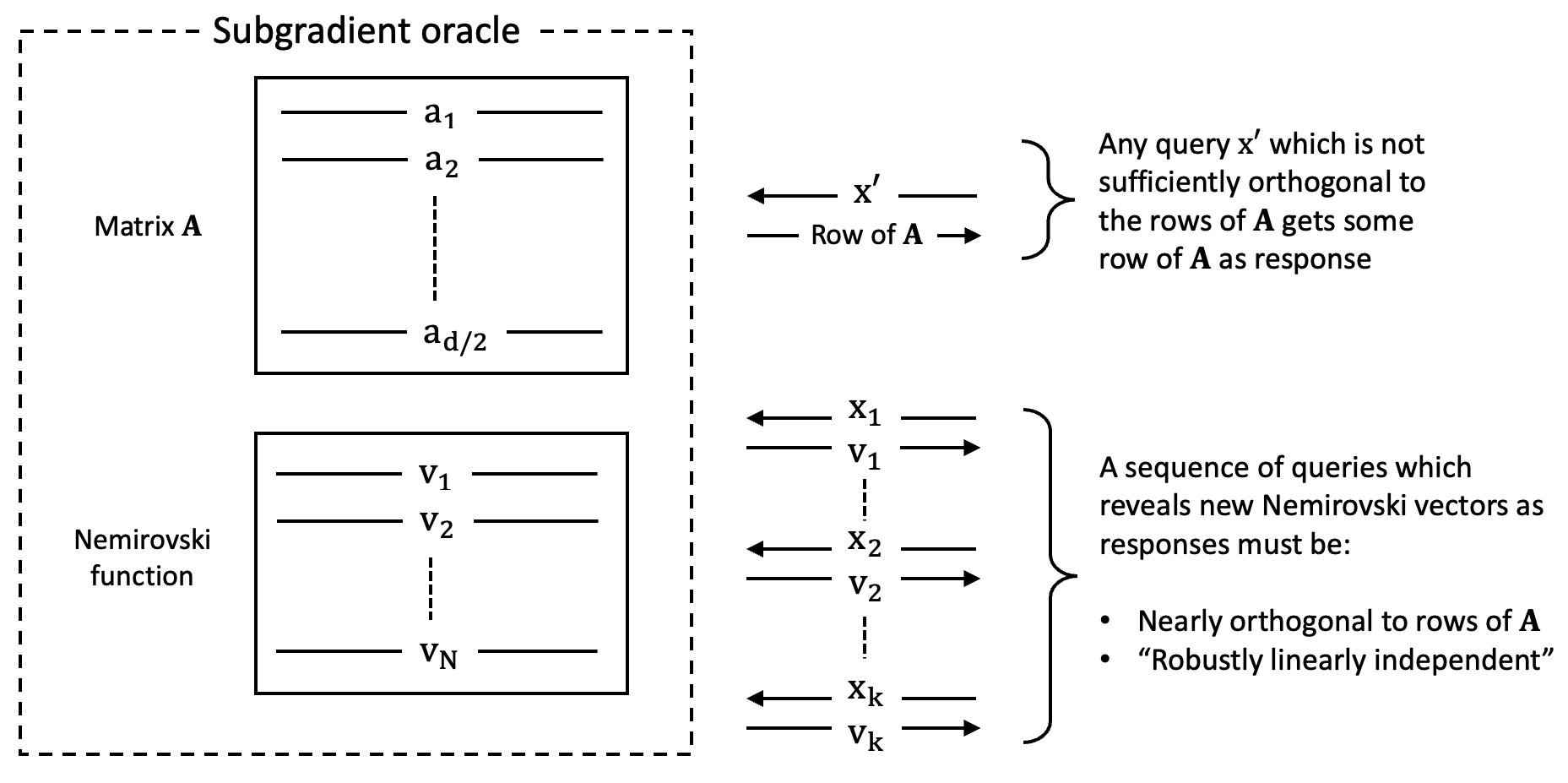}
    \caption{A high-level overview of our proof approach. The rows of $\A$ and the Nemirovski vectors $\{\v_1,\dots, \v_N\}$ are sampled uniformly at random from the hypercube. $\x_i$ is the first query such that the oracle returns $\v_i$ as the response. We show that this function class is ``memory-sensitive" (Definition \ref{def:memorysensitive}) and has the following properties: (1) to successfully minimize the function, the algorithm must see a sufficiently large number of Nemirovski vectors, (2) to reveal new Nemivoski vectors, an algorithm must make queries which are robustly linearly independent and orthogonal to $\A$. Using these properties, we show that minimizing the function is at least as hard as winning the Orthogonal Vector Game (Definition \ref{def:informal_game})  $\approx N/k$ times.  Specifically, we show memory-query tradeoffs for the Orthogonal Vector Game where the goal is to obtain $k$ vectors.  This tradeoff is then leveraged $\approx N/k$ times to obtain the memory-query tradeoff for the optimization problem.
    }
    \label{fig:proof_pic}
\end{figure}

We describe a broad family of optimization problems which may be sensitive to memory constraints (also see Fig. \ref{fig:proof_pic} for an overview). As suggested by the \orthgame\ (Definition \ref{def:informal_game}), the primitive we leverage is that finding vectors orthogonal to the rows of a given matrix requires either large memory or many queries to observe the rows of the matrix. 
With that intuition in mind, let $f : \R^d \to \R$ be a convex function, let $\A \in \R^{n \times d}$, let $\eta$ be a scaling parameter, and let $\rho$ be a shift parameter; define $F_{f, \A, \eta, \rho}(\x)$ as the maximum of $f(\x)$ and $ \eta \linf{\A \x} - \rho $:
\begin{equation}
    \label{eqn:intuitionfunction}
  F_{f, \A, \eta, \rho}(\x) \defeq \max  \left \{ f(\x), \eta \linf{\A \x} - \rho    \right \}.
\end{equation}
We often drop the dependence of $\eta$ and $\rho$ and write $F_{f, \A, \eta, \rho}$ simply as $F_{f, \A}$.  Intuitively, for large enough scaling $\eta$ and appropriate shift $\rho$, minimizing the function $F_{f, \A}(\x)$ requires minimizing $f(\x)$ close to the null space of the matrix $\A$. Any algorithm which uses memory $\Omega(nd)$ can learn and store $\A$ in $\bigo(d)$ queries so that all future queries are sufficiently orthogonal to $\A$; thus this memory rich algorithm can achieve the information-theoretic lower bound for minimizing $f(\x)$ roughly constrained to the nullspace of $\A$. 

However, if $\A$ is a random matrix with sufficiently large entropy then $\A$ cannot be compressed to fewer than $\Omega(n d)$ bits. Thus, for $n = \Omega(d)$, an algorithm which uses only memory $o(d^{2-\delta})$ bits for some constant $0 < \delta \leq  1$ cannot remember all the information about $\A$. Suppose the function $f$ is such that in order to continue to observe new information about the function, it is  insufficient to submit queries that belong to some small dimensional subspace of the null space of $\A$. Then a memory-constrained algorithm must relearn enough information about $\A$ in order to find new vectors in the null space of $\A$ and make queries which return new information about $f$. 

To show our lower bound, we set $f$ in \eqref{eqn:intuitionfunction} to be the Nemirovski function:
\begin{align}
    f(\x) = \max_{i \in [\nnem]} \left(\v_i^{\top} \x  - i \gamma \right)
\end{align}
where the vectors $\v_1,\ldots,\v_N$ are unit vectors drawn i.i.d. from the $d$ dimensional scaled hypercube $\v_i \simiid \mathsf{Unif} ( d^{-1/2}  \mathcal{H}_d )$ and we refer to them as ``Nemirovski vectors.''
With this choice of $f$, the overall function $F_{f, \A}(\x)$ has certain ``memory-sensitive" properties. In particular, to reveal new Nemirovski vectors an algorithm cannot make queries which lie in a low-dimensional subspace. In addition, because of the  $\linf{\A \x}$ term in the definition of $F_{f, \A}(\x)$, queries which reveal new Nemirovski vectors must also be sufficiently orthogonal to $\A$. Together, these properties imply that a sequence of queries which reveals a new set of Nemirovski vectors must also be winning queries for the Orthogonal Vector Game (Definition~\ref{def:informal_game}).  This allows us to leverage our information-theoretic memory-query tradeoffs for the Orthogonal Vector Game. We show that the  algorithm must reveal a sufficiently large number of Nemirovski vectors to optimize $F_{f, \A}(\x)$, therefore we can repeatedly apply the  lower bound for the Orthogonal Vector Game to show our final lower bound (also see Fig. \ref{fig:proof_pic} for a visual overview of the approach).

\subsection{Proof Components}
In summary of \cref{sec:proof_strategy}, there are two critical features of $F_{f,\A}$ which lead to a difficult optimization problem for memory-constrained algorithms. First, the distribution underlying random instances of $\A$ must ensure that $\A$ cannot be compressed. Second, the function $f$ must be such that queries which belong to a small dimensional subspace cannot return too many informative subgradients. We formalize these features by defining a \emph{memory-sensitive base} and a \emph{\functionclassname}. We show that if $\A$ is drawn uniformly at random from a memory-sensitive base and if $f$ and a corresponding first-order oracle are drawn according to a {\functionclassname} then the resulting distribution over $F_{f, \A}(\x)$ and its first-order oracle provides a difficult instance for memory-constrained algorithms to optimize.
\begin{definition}[$(\roundlen$, $\msbconstant$)-memory-sensitive base]
\label{def:memorysensitivebase}
Consider a sequence of sets $\left \{ B_d \right \}_{d>0}$ where $B_d \subset \R^d$. We say $\left \{ B_d \right \}$ is a \emph{$(\roundlen$, $\msbconstant$)-memory-sensitive base} if for any $d$, $\abs{B_d} < \infty$ and for $\h \sim \mathsf{Unif}(B_d)$ the following hold: for any matrix $\mat{Z} = \left( \vec{z}_1, \dots, \vec{z}_{\nreturn} \right) \in \R^{d \times \nreturn}$ with orthonormal columns,
\begin{equation}
 \mathbb{P}\left(\norm{\h}_2 > d \right) \leq 2^{-d}  \qquad \textrm{and} \qquad \mathbb{P} \left( \linf{\mat{Z}^{\top} \vec{h}} \leq 1/2 \right) \leq 2^{- \msbconstant \nreturn}.
\end{equation}
\end{definition}
Setting $n = \lfloor d/2 \rfloor$, we use $\A \sim \mathsf{Unif}(B_d^n)$ to mean $\A = \left( \vec{a}_1, \dots, \vec{a}_n \right)^{\top}$ where each $\vec{a}_i \sim \mathsf{Unif}(B_d)$. Theorem~\ref{thm:memorybase} shows that the sequence of hypercubes is a memory-sensitive base, which is proven in \cref{sec:memory_sensitive_proof}.
\begin{restatable}{theorem}{hypercube}
\label{thm:memorybase}
Recall that $\mathcal{H}_d = \pmset^d$ denotes the hypercube. The sequence $\left \{ 
\mathcal{H}_d \right \}_{d = d_0}^{\infty}$ is a $(k,c_{\mathcal{H}})$ memory-sensitive base for all $k\in [d]$ with some absolute constant $c_{\mathcal{H}}>0$. 
\end{restatable}
Next we consider distributions over functions $f$ paired with some subgradient oracle $\g_f$. Given $f$ and some $\A \sim \mathsf{Unif}(B_d^n)$ we consider minimizing $F_{f, \A}$ as in \cref{eqn:intuitionfunction} with access to an induced first-order oracle:
\begin{definition}[Induced First-Order Oracle $\oracle_{F_{f, \A}}$]
\label{def:subgrad_F}
Given a subgradient oracle for $f$, $\g_f:\R^d \to \R^d$,  matrix $\A \in \R^{n \times d}$, and parameters $\eta$ and $\rho$, let
\begin{gather}
   \g_{\A}(\x)  \defeq 
     \vec{a}_{i_{\min}}, \textrm{ where } i_{\min} \defeq \min \left \{ i \in [n] \textrm{ s.t. }  \vec{a}_i^{\top} \x  = \linf{\A \x} \right \}, \\
   \g_{F_{f, \A}}(\x) \defeq \begin{cases}
    & \g_{\A}(\x), \textrm{ if } \eta \linf{\A \x} - \rho \ge  f(\x), \\
    & \g_f(\x), \textrm{ else.}
    \end{cases}
\end{gather}
The \emph{induced
first-order oracle} for $F_{f, \A}(\x)$, denoted as $\oracle_{F_{f, \A}}$ returns the pair $(F_{f, \A}(\x),\g_{F_{f, \A}}(\x))$.
\end{definition}
We also define \emph{informative subgradients}, formalizing the intuition described at the beginning of this section. 
\begin{definition}[Informative subgradients]\label{def:unique_grad}
Given a query sequence $\seq(\alg, \oracle_{F_{f, \A}}) = \left \{ (\x_i, \g_i) \right \}_{i=1}^{T}$ we construct the sub-sequence of \emph{informative subgradients} as follows: proceed from $i = 1, \dots, T$ and include the pair $(\x_i, \g_i)$ if and only if $(1)$ $f(\x_{i}) > \eta \linf{\A \x_{i}} - \rho$ and $(2)$ no pair $(\x,\g)$ such that $\g=\g_i$ has been selected in the sub-sequence so far. If $(\x_i, \g_i)$ is the $\nth{j}$ pair included in the sub-sequence define $t_j=i$ and we call $\x_i$ the $\nth{j}$ informative query and $\g_i$ the $\nth{j}$ informative subgradient.
\end{definition}

We can now proceed to the definition of a $(L, \memsize, \roundlen, \epstarget)$-{\functionclassname}. 

\begin{definition}[$(L, \nnem, \roundlen, \epstarget)$-\functionclassname]
\label{def:memorysensitive}
Let $\functionclass$ be a distribution over functions $f: \R^d \to \R$ paired with a valid subgradient oracle $\g_f:\R^d \to \R^d$. For Lipschitz constant $L$, ``depth'' $N$, ``round-length'' $\roundlen$, and ``optimality threshold" $\epstarget$, we call $\functionclass$ a $(L, \nnem, \roundlen, \epstarget)$-\emph{\functionclassname} if one can sample from $\functionclass$ with at most $2^d$ uniformly random bits, and there exists a choice of $\eta$ and $\rho$ (from Eq.~\ref{eqn:intuitionfunction})
 
such that for any $\A \in \R^{n \times d}$ with rows bounded in $\ell_2$-norm by $d$ and any (potentially randomized) unbounded-memory algorithm $\algrand$ which makes at most $Nd$ queries: if $(f, \g_f) \sim \functionclass$ then with probability at least $\successprob$ (over the randomness in $(f,\g_f)$ and $\algrand$) the following hold simultaneously:
\begin{enumerate}
\item \textbf{Regularity}: $F_{f, \A, \eta, \rho}$ is convex and $L$-Lipschitz.
 \item \textbf{Robust independence of informative queries}: If $\left \{ \x_{t_j} \right \}$ is the sequence of informative queries generated by $\algrand$ (as per \cref{def:unique_grad}) and
    $S_{j} \defeq \spcommand \left( \{\x_{t_i}: \max( 1, j-\roundlen) \le i \le j \} \right)$ then $\forall j \in [N]$  
  \begin{equation}
      \norm{\proj_{S_{j-1}}({\x}_{t_j})}_2/\norm{{\x}_{t_j}}_2 \leq 1 - 1/d^2. \label{eq:memorysensitive2}
  \end{equation}
    \item \textbf{Approximate orthogonality}: Any query $\x$ with $F_{f, \A, \eta, \rho}(\x) \neq \eta \linf{\A \x}- \rho$ satisfies
    \begin{equation}
        \g_{F_{f, \A}}(\x) = \v_1 \textrm{ or } \linf{\A \x}/\norm{\x}_2\leq \thetathreshold.\label{eq:memorysensitive1}
    \end{equation}
\item \textbf{Necessity of informative subgradients}: If $r < N$ then for any $i \leq t_r$ (where $t_r$ is as in Definition~\ref{def:unique_grad}),
\begin{equation}
    F_{f,\A,\eta, \rho}(\x_i) - F_{f,\A, \eta, \rho}(\x \opt) \geq \epstarget. \label{eq:memorysensitive3}
\end{equation}
\end{enumerate}
\end{definition}

Informally, the \textbf{robust independence of informative queries} condition ensures that any informative query is robustly linearly independent with respect to the previous $\roundlen$ informative queries. The \textbf{approximate orthogonality} condition ensures that any query $\x$ which reveals a gradient from $f$ which is not $\v_1$ is approximately orthogonal to the rows of $\A$. Finally, the \textbf{necessity of informative subgradients} condition ensures that any algorithm which has queried only $r < N$ informative subgradients from $f$ has optimality gap at least $\epstarget$. 

The following construction, discussed in the introduction, is a concrete instance of a $(d^6, \nnem,$ $k$, $1/(20 \sqrt{\nnem}))$-\functionclassname\. 
Theorem~\ref{thm:explicit} proves that the construction is memory-sensitive, and the proof appears in Section \ref{sec:memory_sensitive_proof}.

\begin{definition}[Nemirovski class]\label{def:nemi_func_grad}
For a given $\gamma > 0$ and for $i \in [\nnem]$ draw $ \v_i \simiid \mathsf{Unif} \left(d^{-1/2} \mathcal{H}_d \right) $ and set
\ifcoltshort
$f(\x) \defeq \max_{i \in [\nnem]} \v_i^{\top} \x - i \gamma$ and $\g_f(\x) = \v_{i_{\min}}$,
\else
\begin{equation}
    f(\x) \defeq \max_{i \in [\nnem]} \v_i^{\top} \x - i \gamma \qquad \textrm{ and } \qquad \g_f(\x) = \v_{i_{\min}},
\end{equation}
\fi
where $i_{\min} \defeq \min \{ i \in [\nnem] \textrm{ s.t. }$ $\v_i^{\top} \x - i \gamma = f(\x) \}$. Let $\nem_{\nnem, \gamma}$ be the distribution of $(f,\g_f)$, and for a fixed matrix $\A$ let $\nem_{\nnem, \gamma, \A}$ be the distribution of the induced first-order oracle in Definition \ref{def:subgrad_F}.
\end{definition}

\begin{theorem}
\label{thm:explicit}
For large enough dimension $d$ there is an absolute constant $c > 0$ such that for any given $\roundlen$, the Nemirovski function class with $\gamma = (400 \roundlen \log(d) / d)^{1/2}$, $L = d^6$, $\nnem = c (d/ ( \roundlen \log d) )^{1/3} $, and $\epstarget =1/(20 \sqrt{ \nnem})$ is a  $(L, \nnem, \roundlen, \epstarget)$-\functionclassname.
\end{theorem}

Note that choosing $\nreturn = \lceil 60 \memsize/ (c_{\mathcal{H}} d) \rceil$ gives a $(d^6, \nnem,$ $\lceil 60 \memsize/(c_{\mathcal{H}}d) \rceil,$ $1/(20 \sqrt{\nnem}))$-\functionclassname\ with $\nnem = \frac{1}{205} \left( \frac{ c_{\mathcal{H}} d^2}{\memsize \log d }\right)^{1/3} \approx (d^2/\memsize)^{1/3}$.

\subsection{Main Results}
\label{sec:mainresults}
With the definitions from the previous section in place we now state our main technical theorem for proving lower bounds on memory-constrained optimization.

\begin{theorem}[From memory-sensitivity to lower bounds]
\label{thm:lowerboundprecursor}
For $d \in \mathbb{Z}_{>0}$, $\memsize \in \mathbb{Z}_{\geq d}$, and $\epsilon \in \R_{> 0}$,  let $\A \sim \mathsf{Unif}(B_d^{\lfloor d/2 \rfloor })$ for $(\roundlen, \msbconstant)$-memory-sensitive base
$\left \{ B_d \right \}_{d > 0}$, and let $(f, \g_f) \sim \functionclass$ for $(L, N, \roundlen, \epsilon)$-{\functionclassname} $\functionclass$, with $\roundlen \geq \lceil 60 \memsize / (\msbconstant d)\rceil$. Any $\memsize$-bit memory-constrained randomized algorithm (as per Definition~\ref{def:memoryalgorithmrand}) which outputs an $\epsilon$-optimal point for $F_{f, \A}$ with probability at least $\successprobalg$ requires $\Omega(c_B d^2 N / M)$ queries to the first-order oracle, $\oracle_{F_{f, \A}}$ (where $F_{f,\A}$ and $\oracle_{F_{f, \A}}$ are as defined in Definition \ref{def:subgrad_F}).
\end{theorem}

 \cref{thm:lowerboundprecursor} is agnostic to the particular memory-sensitive class and instead describes how any given memory-sensitive class can be used to derive a query complexity lower bound. The ultimate bound derived using a particular memory sensitive class depend on the classes' relationship between Lipschitz constant $L$, depth $\nnem$, and optimality threshold $\varepsilon$.

In \cref{thm:explicit} we exhibit a $(d^6, \nnem, \lceil 60 \memsize/(c_{\mathcal{H}}d) \rceil,$ $1/(20 \sqrt{\nnem}))$-\functionclassname . 
 For any given $L$, {by scaling this function class by $L/d^6$,} this implies the existence of an $(L, \nnem, \lceil 60 \memsize/(c_{\mathcal{H}}d) \rceil, L/(20 d^6 \sqrt{\nnem}))$-\functionclassname . \cref{thm:lowerboundprecursor}, {along with the existence of a memory-sensitive base from \cref{thm:memorybase},  then} implies a lower bound of $\Omega(c_B d^2 N / M)$ first-order oracle queries. 
 { Setting} $\varepsilon = L/(20 d^6 \sqrt{\nnem})$ and recalling that {{\cref{thm:explicit}}} requires 
 $N \leq (c d^2/M \log(d) )^{1/3}$ for some absolute constant $c > 0$, we obtain the following {memory-query tradeoff parameterized by $\epsilon$}: \emph{ For any $\varepsilon \in [LM^{1/6}/(16d^{19/3}), L/(4 \sqrt{d})]$, any $\memsize$-bit memory-constrained randomized algorithm (as per Definition~\ref{def:memoryalgorithmrand}) which outputs an $\epsilon$-optimal point with probability at least $\successprobalg$ requires
 {$\Omega(\msbconstant L^2d^2/(16^2 d^{12} \varepsilon^2 M))$} first-order oracle queries.} 
 { In \cref{thm:clean}, we explicitly choose $\varepsilon$ as a function of $M$ and $d$ to yield a result that is easy to parse.}

\begin{proof}[Proof of \cref{thm:clean}]
Consider $\memsize$-bit memory-constrained (potentially) randomized algorithms (as per Definition ~\ref{def:memoryalgorithmrand}) where $\memsize$ can be written in the form $d^{1.25 - \delta}$ for some $\delta \in [0,\sfrac{1}{4}]$. By Theorem~\ref{thm:memorybase}, $\left \{ 
\mathcal{H}_d \right \}_{d = d_0}^{\infty}$ is a $(k,c_{\mathcal{H}})$ memory-sensitive base for all $k \in [d]$, where $c_{\mathcal{H}}>0$ is an absolute constant. By Theorem~\ref{thm:explicit}, for some absolute constant $c > 0$ and for $d$ large enough and any given $\roundlen$, if $\gamma =\sqrt{400 \roundlen \log d/d}$, $\nnem = c (d/(\roundlen \log d))^{1/3}$, and $\epsilon = 1/(20 d^6 \sqrt{\nnem}) {\ge 1/d^7}$ (since $N\le d$), the Nemirovski function class from Definition~\ref{def:nemi_func_grad} is $(1, \nnem, \roundlen, \epsilon) $-memory-sensitive (where we rescaled by the Lipschitz constant $1/d^6$). Set $\roundlen = \lceil 60 \memsize / ( c_{\mathcal{H}} d ) \rceil$. Let $(f, \g_f) \sim \nem_{\gamma, \nnem}$ and $\A \sim \mathsf{Unif} \left( \mathcal{H}_d^n \right)$ and suppose $\algrand$ is an $\memsize$-bit algorithm which outputs an $\epsilon$-optimal point for $F_{f, \A}$ with failure probability at most $\failureprobalg$. By \cref{thm:lowerboundprecursor}, $\algrand$ requires at least $c_{\mathcal{H}} \nnem d^2/\memsize \geq c \left( d^2/\memsize \right)^{4/3}(1/\log^{1/3} d)$ many queries. Recalling that $\memsize = d^{1.25 -  \delta}$ for $\delta \in [0, \sfrac{1}{4}]$ results in a lower bound of $\tilde{\Omega}(d^{1+ \frac{4}{3}\delta})$ queries.
\end{proof}

Note that by \cite{woodworth2019open}, the center of mass algorithm can output an $\epsilon$-optimal point for any $L$-Lipschitz function using $\bigo(d^2 \log^2 \left( LB/\epsilon \right) )$ bits of memory and $\bigo \left(d \log (LB/\epsilon) \right)$ first-order oracle queries. Comparing this with the lower bound from \cref{thm:lowerboundprecursor} we see that $\memsize$-bit algorithms are unable to achieve the optimal quadratic memory query complexity with less than $d \nnem/\log(L/\varepsilon)$ bits. Further, \cref{thm:lowerboundprecursor} yields a natural approach for obtaining improved lower bounds.  If one exhibits a \functionclassname\ with depth $N = d$, then  \cref{thm:lowerboundprecursor} would imply that any algorithm using memory of size at most $\memsize = d^{2 - \delta}$ requires at least $\Omega ( d^{1 + \delta} )$ many first-order oracle queries in order to output an $\epsilon$-optimal point. {We note that the subsequent work of \cite{blanchard2023quadratic} shows that a \functionclassname\ with a similar property can be adaptively constructed for any deterministic algorithm, showing that quadratic memory is necessary for deterministic algorithms.}

\begin{remark}[Why $d^{4/3}$: Parameter tradeoffs in our lower bound.]
\label{rem:whyd43}
An idealized lower bound in this setting would arise from exhibiting an $f$ such that to obtain any informative gradient requires $\Theta(d)$ more queries to re-learn $\A$ \emph{and} in order to optimize $f$ we need to observe at least $\Omega(d)$ such informative gradients. Our proof deviates from this ideal on both fronts.  We do prove that we need $\Theta(d)$ queries to get any informative gradients, though with this number we cannot preclude getting $\memsize/d$ informative gradients. Second, our Nemirovski function \emph{can} be optimized while learning only $\bigo((d^{2}/M)^{1/3})$-informative gradients (there are modifications to Nemirovski that increase this number \citep{BubeckJiLeLiSi19}, though for those functions, informative gradients can be observed while working within a small subspace). 
For our analysis we have, very roughly, that $\Omega(d)$ queries are needed to observe $\memsize/d$ informative gradients and we must observe $(d^{2}/\memsize)^{1/3}$ informative gradients to optimize the Nemirovski function. Therefore optimizing the Nemirovksi function requires $\Omega\left(d \right) \times \frac{(d^{2}/M)^{1/3}}{\memsize/d} = (d^2/\memsize)^{4/3}$ many queries; and so when $M = \bigo(d)$ we have a lower bound of $d^{4/3}$ queries.
\end{remark}
\section{Proof Overview}
Here we give the proof of \cref{thm:lowerboundprecursor}. In Section~\ref{sec:orthgame} we present our first key idea, which relates the optimization problem in \eqref{eqn:intuitionfunction} to the \orthgame\ (which was introduced informally in Definition~\ref{def:informal_game}). Next, in Section~\ref{sec:entropy_sketch}, we show that winning this \orthgame\ is difficult with memory constraints.

\subsection{The \orthgame}\label{sec:orthgame}

We formally define the \orthgame\ in Game \ref{alg:gradient_game}. The game proceeds in three parts. In Part 1, the oracle samples a random 
$(d/2\times d)$ dimensional matrix $\A$, and the player gets to observe the matrix and write down an $M$-bit message $\messages$ about the matrix. In Part 2, the player can make queries to the oracle. The oracle responds to any query $\x$ with a row of the matrix $\A$ which has the largest magnitude inner product with $\x$. In Part 3, the player uses   $\messages$ and all the queries and responses from the previous part to find a set of $k$ vectors which are robustly linearly independent, and are roughly orthogonal to all the rows of $\A$. The player wins the game if she can do this successfully.

\IncMargin{3em}
\begin{algorithm2e}[h]
\SetAlgorithmName{Game}{}
  \SetAlgoLined
  \SetAlgoNoEnd
  \DontPrintSemicolon
	\caption{\orthgame}
	    \label{alg:gradient_game}
	    \Indm  
\BlankLine
	    \SetKwInOut{Input}{Input}
		    \Input{$d, \nreturn, m, \base, \memsize$}
		    \Indp
      \SetKwProg{Fn}{Part 1,}{:}{}
      \BlankLine
  \Fn{Store a message}{	
		    \emph{Oracle:} For $n\leftarrow d/2$, sample $\A  \sim \mathsf{Unif}(B_d^n)$.\;
		    \emph{Oracle:} Sample a string $R$ of uniformly random bits of length $3\cdot 2^{d}$ (\playerone\ has read-only access to $R$ throughout the game).\; 
		   \emph{\playeroneshort :} Observe $\mat{A}$ and $R$, and store a $\memsize$-bit message, $\messages$, about the matrix $\A$.\;
		  }
 \SetKwProg{Fn}{Part 2,}{:}{}
      \BlankLine
  \Fn{Make queries}{	
			\For{$i \in [m]$}
			 {
			 \emph{\playeroneshort :} Based on $\messages$, $R$, and any previous queries and responses, use a deterministic algorithm to submit a query $\x_i\in \R^d$.\;
			\emph{Oracle:} As the response to query $\x_i$, return $\g_i =\g_{\A}(\x_i)$ (see Definition \ref{def:subgrad_F}).\;
     }
     }
      \SetKwProg{Fn}{Part 3,}{:}{}
      \BlankLine
  \Fn{Find a good set of vectors}{	
     \emph{\playeroneshort :} Let  $\X$ and $\G$ be matrices with queries $\{\x_i, i \in [m]\}$ and responses $\{\g_i, i \in [m]\}$ as rows, respectively. Based on $(\X, \G,\messages, R)$, use a deterministic function to return vectors $\{\y_1,\dots, \y_{\roundlen}\}$ to the Oracle. Let $\Y$ be the matrix with these vectors as rows.\;
     \capsplayerone\ wins if the returned vectors are \emph{successful}, where a vector $\y_i$ is successful if 
     $ (1) \linf{\mat{A} \y_i}/\norm{\y_i}_2 \leq \thetathreshold$, and
      $(2) \text{ } \proj_{S_{i-1}}(\y_i)/\norm{\y_i}_2 \le 1-1/d^2 \text{ where } S_0\defeq \emptyset\text{, and } S_i \defeq \spcommand(\y_1,\dots, \y_{i})$.\; 
      }
\end{algorithm2e}

Next, we relate winning the \orthgame\ to minimizing $F_{f,\A}$ (\cref{eqn:intuitionfunction}).

\begin{lemma}[Optimizing $F_{f, \A}$ is harder than winning the \orthgame]
\label{lemma:opttogame}
Let $\A\sim \unif(B_d^n)$ and $(f,\g_f) \sim \functionclass$, where $\functionclass$ is a $(L, \nnem, \roundlen, \epstarget)$ \functionclassname. If there exists an $M$-bit algorithm with first-order oracle access that outputs an $\epstarget$-optimal point for minimizing $F_{f, \A}$ using $m \lfloor N/(\nreturn+ 1) \rfloor$ queries with probability at least $\successprobalg$ over the randomness in the algorithm and choice of $f$ and $\A$, then \playerone\ can win the \orthgame\ with probability at least $\sumsuccess$ over the randomness in \playerone's strategy and $\A$. 
\end{lemma}

Here we provide a brief proof sketch before the detailed proof: 
Let $\algrand$ denote a hypothetical $M$-bit memory algorithm for optimizing $F_{f,\A}$ (Definition \ref{def:memoryalgorithmrand}). \capsplayerone\ samples a random function/oracle pair $(f,\g_f)$ from the \functionclassname\ $\functionclass$.  
Then, she uses $\algrand$ to optimize $F_{f, \A}$, and issues queries that $\algrand$ makes to the oracle in the \orthgame. Notice that with access to $\g_f(\x)$ and the oracle's response $\g_{\A}(\x)$, she can implement the sub-gradient oracle $\g_{F_{f, \A}}$ (Definition \ref{def:subgrad_F}). We then use the memory-sensitive properties of $f$ to argue that informative queries made by $\algrand$ must also be successful queries in the \orthgame, and that $\algrand$ must make enough informative queries to find an $\epstarget$-optimal point.

\begin{proof}[Proof of Lemma~\ref{lemma:opttogame}]
In Algorithm~\ref{alg:gradient_game_strategy} we provide a strategy for \playerone\ which uses $\algrand$ (the $M$-bit algorithm that minimizes $F_{f,\A}$) to win the {\orthgame} with probability at least $\sfrac{1}{3}$ over the randomness in $R$ and $\A$.

\begin{algorithm2e}[!h]
\SetAlgorithmName{Algorithm}{}
  \SetAlgoLined
  \SetAlgoNoEnd
  \DontPrintSemicolon
	\caption{\playeroneshort's strategy for the \orthgame}
	    \label{alg:gradient_game_strategy}
	    \Indm  
\BlankLine
	    \SetKwInOut{Input}{Input}
		    \Input{$d, m, \algrand, \functionclass, R$}
		    \Indp
	\SetKwProg{Fn}{Part 1,}{:}{}
  \BlankLine
  \Fn{Strategy to store $\messages$ using $\A$}{	    
 
	Divide random string $R$ into three equal sections $R_1, R_2, R_3$, each of length $2^d$.\;
    Using $R_1$ if needed, sample a function/oracle pair $(f, \g_f) \sim \functionclass$.\;
    \For{$i \in \{1,\ldots, \lfloor N/(  \nreturn + 1) \rfloor \}$}{
   Using $R_2$ if needed, run $\algrand$ to minimize $F_{f,\A}$ until $(i-1)(k+1)$ informative subgradients are observed, or $\algrand$ terminates.\;
   \If{$\algrand$ terminates before $i(k+1)$ informative subgradients are observed}
   {

   \textbf{return} \texttt{Failure}
   }
   Let $\memory_{i}$ be $\algrand$'s memory state. \;
Using $R_3$ if needed, continue running $\algrand$ to minimize $F_{f,\A}$ until $(k+1)$ more informative subgradients are observed, or $\algrand$ terminates. Let $t$ be the total number of first-order queries made by the algorithm in this step.\;
 \If{$\algrand$ terminates before $k+1$ more informative subgradients are observed}
   {

   \textbf{return} \texttt{Failure}
   }
\If{$t \le m$}
	{
	\textbf{return} $\memory_{i}$ as the $\messages$ to be stored.\;
	}
	}
	\textbf{return} \texttt{Failure}
    }  
    	\SetKwProg{Fn}{Part 2,}{:}{}
  \BlankLine
  \Fn{Strategy to make queries}{
  Use $R_1$ to resample 
  $(f, \g_f)$ and set $\algrand$'s memory state to be $\messages$.\;
		\For{$i \in [m]$,}
			 {
			 Using $R_3$ if needed, run $\algrand$ to issue query $\x_i$ to minimize $F_{f, \A}(\x)$.\;
			 Submit query $\x_i$ to the \orthgame\ Oracle to get response $\g_i$.\;
			 \SetKwProg{Fn}{}{:}{}
  \Fn{Simulation of first-order oracle $\oracle_{F_{f, \A}}$}{	
		 $F_{f,\A}(\x_i)=\max(\eta|\g_i^{\top} \x_i|-\rho, f(\x_i))$\;
\lIf{$ F_{f, \A}(\x_i) \ge  f(\x_i)$}
{
$\g_{F_{f, \A}}(\x_i) =\g_{i}$
}
\lElse{
$\g_{F_{f, \A}}(\x_i) =\g_f(\x_i)$
}
}
			Pass $(F_{f,\A}(\x_i),\g_{F_{f, \A}}(\x_i))$ to $\algrand$ to update its state.\;
    }
    }
   
   \SetKwProg{Fn}{Part 3,}{:}{}
  \BlankLine
\Fn{Strategy to find successful vectors based on $(\X, \G, \messages, R)$}{	
    \For{every subset $\{\x_{m_1},\dots,\x_{m_k}\}$ of the queries $\{\x_1,\dots,\x_m\}$}{
    Check if $|\g_{m_j}^{\top} \x_{m_j}|/\norm{\x_{m_j}}\le \thetathreshold \; \forall j \in [k]$\;
    Check if $\{\x_{m_1},\dots,\x_{m_k}\}$ are robustly linearly independent (Def.~\ref{def:memorysensitive})\; 
    \If{the above two conditions are satisfied}{
    \textbf{return} $\{\x_{m_1},\dots,\x_{m_k}\}$ to the Oracle
    }
    }
    \textbf{return} \texttt{Failure}
    }
\end{algorithm2e}

\capsplayerone\ uses the random string $R = (R_1, R_2, R_3)$ to sample a function/oracle pair $(f, \g_f) \sim \functionclass$, and to run the algorithm $\algrand$ (if it is randomized). The first section $R_1$ of the random string is used to sample $(f, \g_f) \sim \functionclass$. Note by Definition~\ref{def:memorysensitive}, $(f, \g_f)$ can be sampled from $\functionclass$ with at most $2^d$ random bits, therefore \playerone\ can sample $(f, \g_f) \sim \functionclass$ using $R_1$. 
The Player uses $R_2$ and $R_3$ to supply any random bits for the execution of $\algrand$, which by Definition~\ref{def:memoryalgorithmrand} of $\algrand$, is a sufficient number of random bits. 

Let $G$ be the event that $R_1,R_2,R_3$, and $\A$ have the property that (1) $\algrand$ succeeds in finding an $\epsilon\opt$-optimal point for
$F_{f,\A}(\x)$, and (2) all the properties of a \functionclassname\ in Definition \ref{def:memorysensitive} are satisfied. By the guarantee in Lemma~\ref{lemma:opttogame}, $\algrand$ 
finds an $\epsilon\opt$-optimal point with failure probability at most $\failureprob$ over the randomness in $\A$, $R_1$ (the randomness in $(f, \g_f)$) and $R_2, R_3$ (the internal randomness used by $\algrand$). Also, the properties in Definition \ref{def:memorysensitive} are satisfied with failure probability at most $\failureprob$ by definition. Therefore, by a union bound, $\Pr[G]\ge \failureprob$. We condition on $G$ for the rest of the proof.

We now prove that Part 1, defined in Algorithm~\ref{alg:gradient_game_strategy}, does not fail under event $G$. Recall the definition of informative subgradients from Definition~\ref{def:unique_grad}, and note that by \cref{eq:memorysensitive3} in Definition~\ref{def:memorysensitive}, 
if $\algrand$ finds an $\epsilon\opt$-optimal point, then $\algrand$ must observe at least $N$ informative subgradients.
Therefore, since $\algrand$ can find an $\epsilon\opt$-optimal point using $m \lfloor N/(  \nreturn + 1) \rfloor$ queries it can observe $N$ informative subgradients from $f$ using $m \lfloor N/(  \nreturn + 1) \rfloor$ queries. Therefore, under event $G$, Part 1 does not fail for any index $i$ in the for loop. Now we show that under event $G$ there must exist some index $i$ such that Part 1 returns $\memory_{i}$ as the $\messages$ to be stored.
 For any execution of the algorithm, let $\vec{v}_i$ denote the $\nth{i}$ informative subgradient from $f$. Block the $\vec{v}_i$'s into $\lfloor N/(\nreturn+ 1) \rfloor$ groups of size $(\nreturn+ 1)$: $\mathsf{Block}_i = \left \{ \vec{v}_{{(i-1)  ( \nreturn + 1) + 1}}, \dots, \vec{v}_{i (\nreturn+ 1)} \right \}$. If $\algrand$ observes $N$ informative subgradients from $f$ using $m \lfloor N/(\nreturn+ 1) \rfloor$ queries, then there is an index $i\opt$ such that $\algrand$ observes $\mathsf{Block}_{i\opt}$ using at most $m$ queries. Therefore, the number of queries to observe $(k+1)$ more informative subgradients is at most $m$ for iteration $i\opt$. Hence Part 1 does not fail under event $G$ and always return $\memory_{i}$ as the $\messages$ for some index $i$.

In Part 2 of the strategy, \playerone\ no longer has access to $\A$, but by receiving the Oracle's responses $\g_i$ she can still implement the first-order oracle $\oracle_{F_{f, \A}}$ (as defined in Definition \ref{def:subgrad_F}) and hence run $\algrand$. Consequently, to complete the proof we will now show that under event $G$ \playerone\ can find a set of successful vectors in Part 3 and win.   
By the guarantee of Part 1, \playerone\ has made at least $ \nreturn + 1$ informative queries among the $m$ queries she made in Part 2. 
By Definition~\ref{def:memorysensitive}, if $\x_i$ is an informative query, then the query satisfies \cref{eq:memorysensitive1} and \cref{eq:memorysensitive2}. Using this, if \playerone\ has observed $ \nreturn + 1$ informative subgradients then she's made at least $\nreturn$ queries which are successful (where the extra additive factor of one comes from the fixed vector $\v_1$ in \cref{eq:memorysensitive1}, and note that the subspace in \cref{eq:memorysensitive2} is defined as the span of all previous $\roundlen$ informative queries, and this will contain the subspace defined in the robust linear independence condition).  
Therefore she will find a set of $k$ successful vectors among the $m$ queries made in Part 2 under event $G$. Since $G$ happens with probability at least $\failureprob$, she can win with probability at least $\failureprob$.
\end{proof}

\subsection{Analyzing the \orthgame : Proof Sketch}\label{sec:entropy_sketch}

With Lemma \ref{lemma:opttogame} in hand, the next step towards proving Theorem~\ref{thm:lowerboundprecursor} is to establish a query lower bound for any memory-constrained algorithm which wins the {\orthgame}. Although our results hold in greater generality, i.e. when $\A \sim \mathsf{Unif}(B_d^n)$ for some memory-sensitive base $\{B_d\}_{d > 0}$, we first provide some intuition by considering the concrete case where the rows of $\A$ are drawn uniformly at random from the hypercube. 
We also consider an alternative oracle model {which we call the \emph{Index-Oracle model}} for which it is perhaps easier to get intuition, and for which our lower bound still holds: suppose \playerone\ can specify any index $i \in [n]$ as a query and ask the oracle to reveal the $\nth{i}$ row of $\A$ (this is the setup in the informal version of the game in Definition~\ref{def:informal_game}). Note that this is in contrast to the oracle in the \orthgame\ which instead responds with the row of the matrix $\A$ which is a subgradient of the query made by the algorithm.
As described in Section \ref{sec:tech_overview}, the instance where $M=\Omega(\nreturn d), m=0$ and the instance where $M=0, m=n$ are trivially achievable for \playerone. We show that these trivial strategies are essentially all that is possible:

 \begin{theorem}[Informal version of Theorem \ref{thm:game_bound}]\label{thm:informal_game}
For some constant $c>0$ and any $k\ge \Omega(\log(d))$, if  $M\le cd\nreturn$ and \playerone\ wins the \orthgame\ with probability at least $1/3$, then $m\ge d/5$. 
 \end{theorem}

Note that when $m=0$,  it is not  difficult to show that \playerone\ requires memory roughly $\Omega(\nreturn d)$ to win the game with any decent probability. This is because each vector which is approximately orthogonal to the rows of $\A$ will not be compressible below $\Omega(d)$ bits with high probability, since rows of $\A$ are drawn uniformly at random from the hypercube. Therefore \playerone\ must use $M\ge \Omega(\nreturn d)$ to store $\nreturn$ such vectors. The main challenge in the analysis is to bound the power of additional, adaptively issued queries. 
In particular, since observing every row $\vec{a}_i$ gives $\Theta(d)$ bits of information about $\A$, and $\nreturn$ successful vectors will only have about $\bigo(\nreturn d)$ bits of information about $\A$, perhaps only observing $\nreturn$ additional rows of $\A$ is enough to win the game. Slightly more concretely, why can't \playerone\ store some $M \ll \nreturn d$ bits of information about $\A$, such that by subsequently strategically requesting  $m\ll n$ rows of $\A$ she now knows $\nreturn$ vectors which are approximately orthogonal to the rows of $\A$? Our result guarantees that such a strategy is not possible.

We now give some intuition for our analysis. 
The underlying idea is to calculate the mutual information between $\A$ and $\mat{Y}$ conditioned on a fixed random string $R$ and the information $\G$ gleaned by \playerone\ (in our proof we slightly augment the matrix $\G$ and condition on that, but we ignore that detail for this proof sketch). We will use the notation $I(\X; \mat{Y} \lvert \Z)$ to denote the mutual information between $\X$ and $\mat{Y}$ conditioned on $\Z$ and $H(\X)$ to denote the entropy of $\X$ (see \cite{CoverTh91}). Our analysis relies on computing a suitable upper bound and lower bound for $I(\A; \mat{Y} \lvert \G, R)$. We begin with the upper bound. We first argue that $\Y$ can be computed just based on $\G, \messages$ and $R$ (without using $\X$),
\begin{align}
    Y = g(\G, \messages, R), \text{for some function $g$.}
\end{align}
The proof of this statement relies on the fact that \playerone's strategy is deterministic conditioned on the random string $R$. Now using the data processing inequality \citep{CoverTh91} we can say, 
\begin{equation}
    I(\A; \mat{Y} \lvert \G, R) \leq I(\A; \G, \messages, R \lvert \G, R) = I(\A ; \messages \lvert \G, R).
\end{equation}
Next we use the definition of mutual information, the fact that conditioning only reduces entropy, and Shannon's source coding theorem \citep{CoverTh91} to write, 
\begin{gather}
    I(\A ; \messages \lvert \G, R) \leq H(\messages \lvert \G, R) \leq H(\messages) \leq \memsize\\
\label{eqn:upperboundentropy}
    \implies I(\A; \mat{Y} \lvert \G, R) \leq \memsize.\label{eq:intuition1}
\end{gather}
On the other hand, we argue that if $\mat{Y}$ is \emph{successful} as per the \orthgame\ with probability $\sfrac{1}{3}$ then there must be a larger amount of mutual information between $\mat{Y}$ and $\A$ after conditioning on $\G$ and $R$. {Recall that $\Y$ contains $k$ robustly linearly independent vectors in order to win the \orthgame, assume for now that $\Y$ is also orthonormal. Then, we may use the fact that the hypercube is a memory-sensitive base (Definition~\ref{def:memorysensitivebase}) to derive that for some constant $c > 0$}
\begin{align}
 \abs{\left \{ \vec{a} \in \{\pm 1\}^d \textrm{ s.t. } \linf{\mat{Y} \vec{a}} \leq \thetathreshold \right \}} \le 2^{d-c\nreturn}.\label{eq:intuition_base}
\end{align}
Since $\G$ contains $m$ rows of $\A$, only the remaining $(n-m)$ rows of $\A$ are random after conditioning on $\G$. Assume for now that the remaining $(n-m)$ rows are independent of $\G$ (as would be true under the Index-Oracle model defined earlier in this section). Then for these remaining rows, we use \cref{eq:intuition_base} to bound their entropy conditioned on $\Y$: 
\begin{align}
    I(\A; \Y \lvert \G,R) & = H(\A \lvert \G, R) - H(\A \lvert \Y, \G, R)\\ 
    &\ge (n-m)d-(n-m)(d-c\nreturn)= c(n-m)\nreturn  =  c(d/2-m)\nreturn ,\label{eq:intuition3}
\end{align}
where we chose $n = d/2$. Now combining \cref{eqn:upperboundentropy} and \cref{eq:intuition3} we obtain a lower bound on the query complexity, $m \ge (d/2) -  \memsize/(c \nreturn)$.
Thus if $M \ll cd\nreturn$, then $m=\Omega(d)$.

While this proof sketch contains the skeleton of our proof, as also hinted in the sketch it oversimplifies many points. 
For example, note that the definition of a memory-sensitive base in Definition \ref{def:memorysensitivebase} requires that $\Y$ be orthonormal to derive \cref{eq:intuition_base}, but successful vectors only satisfy a robust linear independence property. Additionally, in deriving \cref{eq:intuition3} we implicitly assumed that $\G$ does not reduce the entropy of the rows of $\A$ which are not contained in it. This is not true for the \orthgame, since every response $\g_i$ of the oracle is a subgradient of $\linf{\A\x}$ and therefore depends on \emph{all} the rows of the matrix $\A$. Our full proof which handles these issues and dependencies appears in Section~\ref{sec:entropy}. As noted before, our lower bound holds not only for the subgradient oracle, but also when the oracle response $\g_i$ can be an arbitrary (possibly randomized) function from $\R^d\rightarrow \{ \vec{a}_1, \dots, \vec{a}_n\}$ (for example, the model where \playerone\ can query any row $\vec{a}_i$ of the matrix $\A$). We state the lower bound for the \orthgame\ which is the  main result from Section \ref{sec:entropy} below. 

\begin{restatable}{theorem}{gamebound}\label{thm:game_bound}
Suppose for a memory-sensitive base $\left \{ B_d \right \}_{d = 1}^{\infty}$ with constant $\msbconstant > 0$, $\A\sim \mathsf{Unif}(B_d^n)$. Given $\A$ let the oracle response $\g_i$ to any query $\x_i \in \R^d$ be any (possibly randomized) function from $\R^d\rightarrow \{ \vec{a}_1, \dots, \vec{a}_{d/2}\}$, where $\vec{a}_i$ is the $\nth{i}$ row of $\A$ (note this includes the subgradient response in the \orthgame). Assume $M \geq d \log(2d)$ and set $\nreturn=\left\lceil 60M/(\msbconstant d) \right\rceil$. For these values of $\nreturn$ and $M$, if \playerone\  wins the \orthgame\ with probability at least $\sfrac{1}{3}$, then $\mtotal \geq  d/5$.
\end{restatable}

\subsection{Putting things together: Proof of main theorem}

Here we combine the results of the previous subsections to prove \cref{thm:lowerboundprecursor}. 
Very roughly, by Lemma~\ref{lemma:opttogame}, any algorithm for finding an $\epsilon\opt$-optimal point with probability $\successprobalg$ can be used to win the \orthgame\ with probability $\sumsuccess$. Therefore, combining \cref{thm:game_bound} and Lemma~\ref{lemma:opttogame}, and noting $k \approx  \memsize/d$, any algorithm for finding an $\epsilon$-optimal point with probability $\successprobalg$ must use $m \lfloor \nnem/(\nreturn +1)\rfloor \geq c \nnem d^2/\memsize$ queries. 

\begin{proof}[Proof of \cref{thm:lowerboundprecursor}]
We take $\epsilon=1/(20\sqrt{N})$ throughout the proof. By \cite[Theorem 5]{woodworth2019open}, any algorithm for finding a $\epsilon$-optimal point for a $L$-Lipschitz function in the unit ball must use memory at least $d\log(\frac{L}{2\epsilon})$. Therefore, since $\epsilon\le L/(4\sqrt{d})$ then we may assume without loss of generality that $M \geq d\log(\frac{L}{2\epsilon}) \ge d\log(2\sqrt{d})\ge (d/2)\log(4d)  $. Therefore by \cref{thm:game_bound}, if \playerone\ wins  the \orthgame\ with failure probability at most $\twofailureprob$, then the number of queries $m$ satisfy $m\ge d/5$. By Lemma~\ref{lemma:opttogame}, any algorithm for finding an $\epsilon$-optimal point with failure probability at most $\failureprobalg$ can be used to win the \orthgame\ with probability at least $\sfrac{1}{3}$. Therefore, combining \cref{thm:game_bound} and Lemma~\ref{lemma:opttogame}, any algorithm for finding an $\epsilon$-optimal point with failure probability $\failureprobalg$ must use 
\begin{align}
    m \lfloor \nnem/(\nreturn +1)\rfloor \ge \frac{\msbconstant N d^2}{(5 \cdot 60) M}
\end{align}
queries. Therefore if $\epsilon =  1/(20\sqrt{\nnem})$ and $c =  \msbconstant/(2\cdot 20\cdot 5\cdot 3\cdot20^2)\ge \msbconstant/10^6$, any memory-constrained algorithm requires at least $c d^2/(\epsilon^2 \memsize)$ 
first-order queries. 
\end{proof}


\section{The Hypercube and Nemirovski function class are Memory-Sensitive}\label{sec:memory_sensitive_proof}

In this section we prove that the hypercube and the Nemirovski function class have the stated memory-sensitive properties. The proofs will rely on a few concentration bounds stated and proved in Appendix \ref{sec:concentration}.

\subsection{The Hypercube is a Memory-Sensitive Base}
Here we prove that the hypercube is a memory sensitive basis. Recall Theorem~\ref{thm:memorybase}.
\hypercube*
\begin{proof} [Proof of Theorem~\ref{thm:memorybase}]
First observe that for any $\h \in \mathcal{H}_d$, $\norm{\h}_2 = \sqrt{d}$. Next let $\Z = (\z_1, \dots, \z_k) \in \R^{d \times k}$ be $k$ orthonormal $d$-dimensional vectors and let $\h \sim \mathsf{Unif}(\mathcal{H}_d)$. Note that each coordinate of $\h$ is sub-Gaussian with sub-Gaussian norm bounded by 2 and $\E[\h \h^{\top}]=\id_d$. By Lemma \ref{lemma:subgauss_proj} in Appendix \ref{sec:concentration} there is an absolute constant $\chw > 0$ such that for any $s \geq 0$,
\begin{equation}
    \mathbb{P} \left(   \nreturn - \norm{\mat{Z}^{\top} \h}_2^2  > s \right) \leq  \exp \left( - \chw \min \left \{ \frac{s^2}{16 \nreturn}, \frac{s}{4} \right \} \right).
\end{equation}
Taking $s = \nreturn/2$ we observe,
\begin{equation}
     \mathbb{P} \left(   \norm{\mat{Z}^{\top} \h}_2 < \sqrt{\nreturn/2} \right) \leq  \mathbb{P} \left(  \nreturn - \norm{\mat{Z}^{\top} \h}_2^2 > \nreturn/ 2 \right) \leq  \exp \left( - \chw  \nreturn/ 64 \right).
\end{equation}
Noting that $\norm{\mat{Z}^{\top} \h}_2 \leq \sqrt{\nreturn} \linf{\mat{Z}^{\top} \h}$ we must have that for any $t \leq 1/\sqrt{2}$,
\begin{equation}
    \mathbb{P} \left( \linf{\mat{Z}^{\top} \h} \leq t \right) \leq  \exp \left( - \chw  \nreturn/ 64 \right) \leq 2^{ - c \nreturn},
\end{equation}
for $c = \chw/( 64 \log_e 2)$.
\end{proof}

\subsection{The Nemirovski function class is Memory-Sensitive}\label{sec:nemirovski}
Recall we define $\nem_{\nnem, \gamma}$ to be the distribution of $(f, \g_f)$ where for $i \in [\nnem]$ draw $ \v_i \simiid (1/\sqrt{d}) \mathsf{Unif} \left( \mathcal{H}_d \right) $ and set
\begin{equation}
    f(\x) \defeq \max_{i \in [\nnem]} \v_i^{\top} \x - i \gamma \quad \textrm{ and } \quad \g_f(\x) = \v_{i_{\min}}
    \quad
    \textrm{ for }
    \quad 
    i_{\min} \defeq \min \left \{i \in [\nnem] \textrm{ s.t. } \v_i^{\top} \x - i \gamma = f(\x) \right \}\,.
\end{equation}
We will show that for 
$\gamma=\sqrt{\frac{c \roundlen \log(d) }{d}}$, where $c=400$ is a fixed constant, $\eta = d^5$, $\rho = 1$, if $N \leq (1/32 \gamma)^{2/3}$ then $\nem_{\nnem, \gamma}$ is a $(d^6, \nnem, \nreturn, 1/(20 \sqrt{\nnem}))$-{\functionclassname}. To show the required properties of the Nemirovski function, we define \emph{\playerb}\ which iteratively constructs the Nemirovski function. \capsplayerb\ will be defined via a sequence of $N$ \emph{Phases}, where the $\nth{j}$ Phase ends and the next begins when the algorithm makes a query which reveals the $\nth{j}$ Nemirovski vector. We begin with some definitions for these Phases.

\begin{definition}
    \label{def:nemtoj}
For a fixed $\A$ and vectors $\{\v_i, i \in [j]\}$, the \emph{Nemirovski-induced function for Phase $j$} is 
\begin{equation}
    F_{\nem, \A}^{(j)} (\x)  \defeq \max \left \{ \max_{i \in [j]}\left( \v_i^{\top} \x - i \gamma\right), \eta \linf{\A \x} - \rho  \right \}
\qquad
    \textrm{ with }
\qquad 
    F_{\nem, \A}^{(0)} (\x)  \defeq \eta \linf{\A \x} - \rho.
\end{equation}
\end{definition}

Analogous to Definition \ref{def:nemi_func_grad}, we define the following subgradient oracle,
\begin{align}
   \g_f^{(j)}(\x) \defeq \v_{k_{\textrm{min}}}, \textrm{ where } k_{\textrm{min}} \defeq \min \left \{ i \in [j] \textrm{ s.t. } \v_i^{\top} \x - i \gamma = F_{\nem, \A}^{(j)}(\x) \right \}.\label{eq:nemi_oracle}
   \end{align}
 Following Definition \ref{def:subgrad_F}, this defines a subgradient oracle $\g_{F}^{(j)}(\x)$ for  the induced function  $F_{\nem, \A}^{(j)} (\x)$.
 \begin{align}
   \g_{F}^{(j)}(\x) \defeq \begin{cases}
    & \g_{\A}(\x), \textrm{ if } \eta \linf{\A \x} - \rho \ge  f(\x), \\
    & \g_f^{(j)}(\x), \textrm{ else.}
    \end{cases}
\end{align}
The first-order oracle $\oracle_{F_{\nem, \A}^{(j)}}$ which returns the pair $(F_{\nem, \A}^{(j)}(\x), \g_{F}^{(j)}(\x))$ will be used by \playerb\ in Phase $j$. 
With these definitions in place, we can now define \playerb\ in Algorithm \ref{alg:nem_game}.

 \begin{algorithm2e}[!h]
 \SetAlgorithmName{Algorithm}{}
   \SetAlgoLined
   \SetAlgoNoEnd
   \DontPrintSemicolon
 	\caption{\capsplayerb}
 	    \label{alg:nem_game}
 	    \SetKwInOut{Input}{Input}
 	    \SetKwInOut{Return}{return}
 	    \SetKwFunction{FBegin}{initialize}
 	    \SetKwFunction{FQuery}{query}
 	    \tcp{Initialize the oracle (before calling \FQuery) by storing the input, setting the current phase number ($j =1$), and computing $\v_1$ }
 	    \SetKwProg{fstart}{}{:}{}
 		    \fstart
 		    {\FBegin{$\A \in \R^{n \times d}, N > 0, d >0$}}{
 		        Save $\A \in \R^{n \times d}$, $N > 0$, and $d >0$\;
 		        Set $\v_1\simiid (1/\sqrt{d}) \mathsf{Unif} \left( \mathcal{H}_d \right)$ and $j \gets 1$ \;
 		    }
             \BlankLine

             \tcp{Compute the oracle output for query vector $\x$ (after having called \FBegin)}
 	        \fstart{\FQuery{$\x \in \unitball$}}{
 	            $(\alpha,\v) \gets (F^{(j)}_{\nem, \A}(\x), \g_{F}^{(j)}(\x))$ \tcp{store oracle response} 
 		     \If(\tcp*[h]{Check if time for Phase $j + 1$}){$\g_{F}^{(j)}(\x) = \v_j$}
 		     {
 		    Begin Phase $j + 1$ of \playerbshort\ by setting $\v_{j + 1} \simiid (1/\sqrt{d}) \mathsf{Unif} \left( \mathcal{H}_d \right)$ and $j \gets j + 1$\;
 		     }
 		     \If{final query to oracle and $j < N$}
 		     {
 		     For all $j'>j$, \playerbshort\ draws $\v_{j'}\simiid  (1/\sqrt{d}) \mathsf{Unif} \left( \mathcal{H}_d \right)$.\;
 		     }
 		     \Return{$(\alpha, \v)$}
 		     }
 \end{algorithm2e}

As we prove in the following simple lemma, for any fixed algorithm the distribution over the responses from $\oracle_{F_{\nem, \A}^{(N)}}$ is the same as the distribution over the responses from the oracle $\nem_{\nnem, \gamma, \A}$ for the Nemirovksi function defined in Definition~\ref{def:nemi_func_grad}. Therefore, it will be sufficient to analyze \playerb\ and $\oracle_{F_{\nem, \A}^{(N)}}$.

\begin{lemma}\label{lem:sample_nemi}
The vectors $\{\v_i, i \in [N]\}$ for the final induced oracle $\oracle_{F_{\nem, \A}^{(N)}}(\x)=(F_{\nem, \A}^{(N)}(\x), \g_{F}^{(N)}(\x))$ are sampled from the same distribution as the vectors  $\{\v_i, i \in [N]\}$ for the Nemirovski function  $\nem_{N, \gamma, \A}$  (Definition~\ref{def:nemi_func_grad}). Moreover, for the same setting of the vectors  $\{\v_i, i \in [N]\}$, any fixed matrix $\A$ and any fixed algorithm, the responses of both these oracles are identical.
\end{lemma}
\begin{proof}

The proof follows from definitions. For the first part, note that the Nemirovski vectors in Algorithm~\ref{alg:nem_game} are sampled at random independent of \playera's queries, and from the same distribution as in Definition~\ref{def:nemi_func_grad}. For the second part, note that by definition, for a fixed set of vectors $\{\v_i, i \in [N]\}$ and matrix $\A$, the first-order oracle $\oracle_{F_{\nem, \A}^{(N)}}$ is identical to $\nem_{N, \gamma, \A}$ from Definition \ref{def:nemi_func_grad}. 
\end{proof}

Our main goal now is to prove the following two statements (1) with high probability,  \playerb\ is consistent with itself, i.e. its outputs are the same as the outputs of $\oracle_{F_{\nem, \A}^{(N)}}=(F_{\nem, \A}^{(N)}(\x), \g_{F}^{(N)}(\x))$,  (2) \playerb\ has the properties required from a \functionclassname. To prove these statements, we define an event $E$ in Definition~\ref{defs:e1ande2}, and then argue that the failure probability of $E$ is at most $1/d$. Event $E$ allows us to argue both that \playerb\ is identical to $\oracle_{F_{\nem, \A}^{(N)}}$, as well as to argue that $\nem_{N, \gamma}$ is \functionclassname.

\begin{definition}
\label{defs:e1ande2}
Let $E$ denote the event that $\forall j \in [N]$, the Nemirovski vector $\v_j$ has the following properties:
\begin{enumerate}
    \item For any query $\x$ submitted by \playera\ in Phase 1 through $j$, 
    $|\x^\top \v_j| \le \sqrt{\frac{10\log(d)}{d}}$.
    \item $\norm{\proj_{S_{\roundindex}}(\v_{j})}_2 \leq \sqrt{\frac{30 \roundlen \log(d) }{d}}$ where $S_j$ is defined as in Definition~\ref{def:memorysensitive}.
\end{enumerate}
\end{definition}

\begin{lemma}\label{lem:vector_condition}
Recall the definition of event $E$ in Definition \ref{defs:e1ande2}. Suppose that for every Phase $j\in[N]$, \playera\ makes at most $d^2$ queries in that Phase.  Then $\Pr[E]\ge 1 - (1/d)$.
\end{lemma}

\begin{proof}
We prove for any fixed $j \in [N]$, the failure probability of event $E$ in Phase $j$ is at most $1/d^2$ and then apply a union bound over all $j \in [N]$ (where $N \leq d$). 
Define $T_j$ to be the sequence of the next $d^2$ vectors \playera\ would query in the $\nth{j}$ Phase, where the next query is made if the previous query does not get the new Nemirovski vector $\v_{j}$ as the response (i.e., the previous query does not end the $\nth{j}$ Phase). Since we can think of \playera's strategy as deterministic conditioned on its memory state and some random string $R$ (as in Definition \ref{def:memoryalgorithmrand}), the sequence $T_j$ is fully determined by the function $F_{\nem, \A}^{(j-1)}$, $R$ and the state of \playera\ when it receives $\v_{j-1}$ as the response from \playerb\ at the end of the $\nth{(j-1)}$ Phase. Critically, this sequence in independent of $\v_j$. 
Separately, \playerb\ samples $\v_j \simiid \left( (1/\sqrt{d}) \mathsf{Unif} \left( \mathcal{H}_d \right) \right)$, and therefore $\v_j$ is sampled independently of all these $d^2$ queries in the set $T_j$.

The proof now follows via simple concentration bounds. First consider property (1) of event $E$. By the argument above we see that for any any vector $\x$ which is in the set $T_j$ or was submitted by \playera\ in any of the previous $(j-1)$ Phases, by Corollary \ref{cor:xtopvgaussian}, $\x^{\top} \v_j$ is a zero-mean random variable with sub-Gaussian parameter at most $\norm{\x}_2/\sqrt{d} \leq 1/\sqrt{d}$. Then by Fact \ref{fact:subgauss_conc}, for any $t \geq 0$,
\begin{equation}
 \mathbb{P} \left(\abs{\x^{\top} \v_j} \geq t \right) \leq 2 \exp(-d t^2/2).
\end{equation}
We remark that to invoke Corollary \ref{cor:xtopvgaussian}, we do not condition on whether or not vector $\x$ will ultimately be queried by \playera\ (this is critical to maintain the independence of $\x$ and $\v_j$). Picking $t = \sqrt{10 \log(d)/d}$ we have with failure probability at most $2/d^5$,
\begin{equation}
    \abs{\x^{\top} \v_{j}} \leq \sqrt{\frac{10\log(d)}{d}}.
\end{equation}
Since \playera\ makes at most $d^2$ queries in any Phase, there are at most $Nd^2$ vectors in the union of the set $T_j$ of possible queries which \playera\ makes in the $\nth{j}$ Phase, and the set of queries submitted by \playera\ in any of the previous Phases. Therefore by a union bound, property (1) of event $E$ is satisfied by $\v_{j}$ with failure probability at most $2Nd^2/d^5\le 2/d^2$, where we used the fact that $\nnem \leq d$.

We now turn to property (2) of event $E$. Note that $S_{\roundindex}$ depends on $\x_{t_{j}}$, which is the first query for which \playerb\ returns $\v_j$. However, using the same idea as above, we consider the set $T_j$ of the next $d^2$ queries which \playera\ would make, and then do a union bound over every such query. For any query $\x \in T_j$, consider the subspace
\begin{equation}
    S_{j-1,\x} \defeq \spcommand \left(\{\x_{t_i}: j-\roundlen \le i < j, i>0 \}, \x \right).
\end{equation}
Note that $S_j = S_{j-1, \x_{t_j}}$. Recalling the previous argument, $\v_j$ is independent of all vectors in the above set  $S_{j-1,\x}$ . Therefore by Corollary \ref{cor:explicitprojectionbound} there is an absolute constant $c$ such that with failure probability at most $c/d^5$, 
\begin{equation}
     \norm{\proj_{S_{j-1,\x}}(\v_j)}_2 \leq \sqrt{ \frac{30 k \log(d)}{d}}.
\end{equation}

Under the assumption that \playera\ makes at most $d^2$ queries in the $\nth{j}$ Phase, the successful query $\x_{t_{j}}$ must be in the set $T_j$. Therefore, by a union bound over the $d^2$ queries which \playera\ submits,
\begin{equation}
\label{eqn:projectionsmall}
    \norm{ \proj_{S_{\roundindex}} \left( \v_{j} \right) }_2 \leq \sqrt{\frac{30 \roundlen \log(d) }{d}},
\end{equation}
with failure probability at most $c/d^{3} \leq 1/d^2$ (for $d$ large enough). Therefore, both properties (1) and (2) are satisfied for vector $\v_j$ with failure probability at most $1/d^{2}$ and thus by a union bound over all $\nnem$ Nemirovski vectors $\v_j$, $E$ happens with failure probability at most $1/d$. 
\end{proof}

Now that we have justified that event $E$ has small failure probability we will argue that conditioned on $E$, \playerb\ returns  consistent first-order information. This is critical to our analysis because we prove memory-sensitive properties through the lens of \playerb. 

\begin{lemma}
 \label{lem:gradient_consistent}
 Under event $E$ defined in Definition~\ref{defs:e1ande2}, \playerb's responses are consistent with itself, i.e., for any query made by \playera, \playerb\ returns an identical first-order oracle response to the final oracle $(F^{(N)}_{\nem, \A}(\x), \g^{(N)}_F(\x))$.
\end{lemma}

\begin{proof}
Assume event $E$ and fix any $j \in [\nnem-1]$. Let $\x$ be any query \playera\ submits sometime before the end of the $\nth{j}$ Phase. Under event $E$ for any $j' > j$, 
$$|\x^\top \v_j| \le \sqrt{\frac{10\log(d)}{d}}, \quad |\x^\top \v_{j'}| \le \sqrt{\frac{10\log(d)}{d}}.$$
Therefore, since $\gamma=\sqrt{c\roundlen\log(d)/d}$, for any $c>40$ and any $j'>j$,
\begin{equation}
\x^\top (\v_{j'}-\v_j) <  \gamma.
\end{equation}
In particular this implies that for $j' > j$
$$ \x^\top \v_j - j\gamma  >  \x^\top \v_{j'} - j'\gamma.$$
Therefore for all $j' > j$, $\v_{j'}$ is not the subgradient $\g^{(\nnem)}_F(\x)$. This implies that in the $\nth{j}$ Phase, \playerb\ always responds with the correct function value $F^{(j)}_{\nem, \A}(\x) = F^{(N)}_{\nem, \A}(\x)$ and subgradient $\g^{(j)}_F(\x) = \g^{(\nnem)}_F(\x)$. 
\end{proof}

Now we address the memory-senstive properties of the function class by analyzing \playerb. Here, event $E$ is crucial since it ensures that the Nemirovski vectors lie outside the span of previous queries which returned a unique gradient.

\begin{remark}
Recall from Definition \ref{def:unique_grad} that $\x_{t_{j}}$ is the query which returns the $\nth{j}$ informative subgradient from the function $f$. In terms of the Nemirovski function and \playerb, we note that  $\x_{t_{j}}$ is the query which ends Phase $j$ of \playerb\ (the first query such that \playerb\ returns  $\v_j$ as the gradient in response).
\end{remark}

\begin{lemma}\label{lem:highrank}
Recall the definition of $\x_{t_j}$ and $S_{\roundindex}$ from  \cref{def:unique_grad,def:memorysensitive} and that $\gamma=\sqrt{\frac{c \roundlen \log(d) }{d}}$. If $c \geq 400$, then under event $E$ from Definition~\ref{defs:e1ande2},
\begin{equation}
    \frac{\norm{\proj_{S_{j-1}}(\x_{t_j} )}_2}{\norm{\x_{t_j} }_2 }\leq 1 - (1/d^{\projpower}).
\end{equation}
\end{lemma}

\begin{proof}
We prove Lemma~\ref{lem:highrank} by providing both an upper and lower bound on $ \x_{t_{j}}^{\top}(\v_{j}-\v_{j-1}) $ under event $E$. Towards this end, observe that for any $\x \in S_{j-1}$,
\begin{equation}
    \label{eqn:innerprod}
    \abs{\x^{\top} \v_{j-1}} \le \norm{\x}_2 \norm{\proj_{S_{\roundindex-1}}\left(\v_{j-1}  \right)}_2.
\end{equation}
Therefore, recalling that every query $\x$ satisfies $\norm{\x}_2\le 1$, we see that under $E$,
\begin{equation}
    \abs{\proj_{S_{j-1}}(\x_{t_j})^{\top}\v_{j-1}} \leq  \sqrt{\frac{30 \roundlen \log(d)}{d}}.
\end{equation}
Next, writing $\x_{t_j}$ as  $\x_{t_j} = \proj_{S_{j-1}}(\x_{t_j}) + \proj_{S_{j-1}^{\perp}}(\x_{t_j})$ and recalling that $\norm{\x_{t_j}}_2 \leq 1$ we observe that
  \begin{equation}
      \norm{\proj_{S_{j-1}^{\perp}}(\x_{t_j})}_2^2 = \norm{\x_{t_j}}_2^2 - \norm{\proj_{S_{j-1}}(\x_{t_j})}_2^2 \leq  1 - \norm{\proj_{S_{j-1}}(\x_{t_j})}_2^2/\norm{\x_{t_j}}_2^2.
 \end{equation}
 Recalling that $\norm{\v_j}_2 = 1$,
\begin{equation}
    \abs{\proj_{S_{j-1}^{\perp}}(\x_{t_j})^{\top} \v_{j-1}} \leq \norm{\proj_{S_{j-1}^{\perp}}(\x_{t_j})}_2 \norm{\v_{j-1}}_2 \leq  \left( 1 - \norm{\proj_{S_{j-1}}(\x_{t_j})}_2^2/\norm{\x_{t_j}}_2^2\right)^{1/2}.
\end{equation}
We also note that under event $E$,
 $$|\x_{t_j}^\top \v_j| \le \sqrt{\frac{10\log(d)}{d}}.$$
Therefore, 
\begin{align}
    \x_{t_{j}}^{\top}(\v_{j}-\v_{j-1}) & \le  \abs{\x_{t_{j}}^{\top}\v_{j}} + \abs{ \proj_{S_{j-1}}(\x_{t_j})^{\top}\v_{j-1}} + \abs{\proj_{S_{j-1}^{\perp}}(\x_{t_j})^{\top} \v_{j-1}}  \\
    & \leq \sqrt{\frac{10\log(d)}{d}} +  \sqrt{\frac{30 \roundlen \log(d) }{d}} +  \sqrt{1 - \norm{\proj_{S_{j-1}}(\x_{t_j})}_2^2/\norm{\x_{t_j}}_2^2} \\
    & \le \sqrt{\frac{80 \roundlen \log(d) }{d}} + \sqrt{1 - \norm{\proj_{S_{j-1}}(\x_{t_j})}_2^2/\norm{\x_{t_j}}_2^2}.
\end{align}

Note that $\v_j$ is an informative gradient and therefore $\v_j \in \partial \nem(\x_{t_j})$. Thus
\begin{align}
     \x_{t_{j}}^{\top}\v_{j-1}- (j-1) \gamma &\leq \x_{t_j}^{\top}\v_{j}  - j \gamma, \qquad \textrm{ or equivalently} \qquad \x_{t_j}^{\top}\left( \v_j - \v_{j-1} \right) \geq \gamma.
\end{align}
Therefore,
\begin{align}
    \sqrt{\frac{80 \roundlen \log d }{d}} + \sqrt{1 - \norm{\proj_{S_{j-1}}(\x_{t_j})}_2^2/\norm{\x_{t_j}}_2^2} &\geq \gamma\\
    \implies \sqrt{1 - \norm{\proj_{S_{j-1}}(\x_{t_j})}_2^2/\norm{\x_{t_j}}_2^2} &\ge \gamma-\sqrt{\frac{80 \roundlen \log d }{d}} ,\\
    \implies {1 - \norm{\proj_{S_{j-1}}(\x_{t_j})}_2^2/\norm{\x_{t_j}}_2^2} &\ge \left(\gamma-\sqrt{\frac{80 \roundlen \log d }{d}}\right)^2 ,
\end{align}
where the last step uses the fact that $\gamma=\sqrt{\frac{c \roundlen \log(d) }{d}}$ for $c\geq 400$.
Rearranging terms and using that $(1 - x)^{1/2} \leq 1 - (1/4)x$ for $x \geq 0$ gives us
\begin{equation}
    \norm{\proj_{S_{j-1}}(\x_{t_j})}_2/\norm{\x_{t_j}}_2  \leq \sqrt{1 -  \left( \gamma - \sqrt{\frac{80 \roundlen \log d }{d}} \right)^2} \leq 1 - \frac{1}{4} \left( \gamma - \sqrt{\frac{80 \roundlen \log d }{d}} \right)^2.
\end{equation}
Plugging in $\gamma=\sqrt{\frac{c \roundlen \log(d) }{d}}$, for $c\geq 400$ gives us the claimed result for $d$ large enough,
\begin{equation}
     \norm{\proj_{S_{j-1}}(\x_{t_j})}_2/\norm{\x_{t_j}}_2  \leq 1 - \frac{1}{4} \frac{\roundlen \log d}{d} \left( \sqrt{400} - \sqrt{80} \right)^2 \leq 1 - \frac{1}{d^2}.
\end{equation}
\end{proof}
We next turn our attention to the optimality achievable by an algorithm which has seen a limited number of Nemirovski vectors. Lemma \ref{lem:subopt} lower bounds the best function value achievable by such an algorithm.

\begin{lemma}\label{lem:subopt}
Recall that $\gamma=\sqrt{\frac{c \roundlen \log(d) }{d}}$ and suppose $c > 10$. Assume \playera\ has proceeded to Phase $r$ of \playerb\ and chooses to not make any more queries. Let $Q$ be the set of all queries made so far by \playera. Then under event $E$ defined in 
Definition \ref{defs:e1ande2},
$$\min_{ \x \in Q} F_{\nem, \A}^{(N)}(\x) \ge -(r+1)\gamma.$$
\end{lemma}
\begin{proof}
We first note that 
 $$\min_{ \x \in Q} F_{\nem, \A}^{(N)}(\x) \ge \min_{ \x \in Q} F_{\nem, \A}^{(r)}(\x).$$
Now under event $E$, all queries $\x \in Q$ made by \playera\ satisfy,
 $$|\x^\top \v_r| \le \sqrt{\frac{10\log(d)}{d}}.$$
Assuming event $E$, by Lemma \ref{lem:gradient_consistent} if \playerb\ responds with some Nemirovski vector as the subgradient, since this Nemirovski vector must be some $\v_i$ for $i < r$ and since the returned subgradients are valid subgradients of the final function $F_{\nem, \A}^{(\nnem)}$, the Nemirovski vector $\v_r$ could not have been a valid subgradient of the query. Therefore, all queries $\x\in Q$ must satisfy,
 \begin{align}
     F_{\nem, \A}^{(r)}(\x) \ge \x^\top \v_r -r\gamma \ge -\sqrt{\frac{10\log(d)}{d}} -r\gamma\ge -(r+1)\gamma 
 \end{align}
 where in the last inequality we use the fact that $\gamma\ge \sqrt{\frac{10\log(d)}{d}}$.
\end{proof}
Lemma \ref{lemma:optimalvalue} upper bounds the minimum value of the final function $F_{\nem, \A}^{(N)} (\x)$, and will be used to establish an optimality gap later.
\begin{lemma}
\label{lemma:optimalvalue}
Recall that $\gamma=\sqrt{\frac{c \roundlen \log(d) }{d}}$ and suppose $c > 3$. For any $\A\in \R^{n\times d}$ where $n\le d/2$ and sufficiently large $d$, with failure probability at most $2/d$ over the randomness of the Nemirovski vectors $\{\v_j, j \in [N]\}$,
\begin{equation}
\label{eq:opt}
    \min_{\norm{\x}_2 \le 1} F_{\nem, \A}^{(N)} (\x) \leq  - \frac{1}{8 \sqrt{\nnem}}. 
\end{equation}
\end{lemma}

\begin{proof}
Let the rank of $\A$ be $r$. Let $\Z \in \R^{d \times (d-r)}$ be an orthonormal matrix whose columns are an orthonormal basis for the null space $A^{\perp}$ of $\A$. We construct a vector $\xopt$ which (as we will show) attains $F_{\nem, \A}(\xopt) \leq -1/8\sqrt{\nnem}$ as follows, 
\begin{equation}
    \xopt = \frac{-1}{2 \sqrt{\nnem }} \sum_{i \in [\nnem]} \pap{i} = \frac{-1}{2 \sqrt{\nnem }} \sum_{i \in [\nnem]} \Z\Z^{\top} \v_i =  \Z\Z^{\top} \left(\frac{-1}{2 \sqrt{\nnem }}\sum_{i \in [\nnem]} \v_i \right).
\end{equation}
Our proof proceeds by providing an upper bound for $F_{\nem, \A}(\xopt)$. First we bound $\norm{\xopt}_2^2$. Let $\bar{\v} =  \left(\frac{-1}{2 \sqrt{\nnem }}\sum_{i \in [\nnem]} \v_i \right)$. Note that
\begin{equation}
   \norm{\xopt}_2^2 =  \norm{\Z^{\top} \bar{\v}}_2^2.
\end{equation}
By Fact \ref{fact:sum+subgauss}, each coordinate of $\bar{\v}$ is sub-Gaussian with sub-Gaussian norm $\subgaussconst/\sqrt{d}$ (where $\subgaussconst$ is the absolute constant  in Fact \ref{fact:sum+subgauss}). Also $\E[\bar{\v} \bar{\v}^{\top}]=(1/4d)\id_d$. Therefore by Lemma \ref{lemma:subgauss_proj}, with failure probability at most $2 \exp(-c_1d) $ for some absolute constant $c_1>0$,
\begin{equation}
    \label{eqn:normbound1}
      \norm{\xopt}_2^2 \leq \frac{d-r}{4d}+ \frac{1}{2} \le 1.
\end{equation}
Therefore $\xopt$ lies in the unit ball with failure probability at most $2 \exp(-c_1d) $.
Next fix any $\v_j \in \left \{ \v_i \right \}_{i \in [\nnem]} $ and consider
\begin{align}
    \v_j^{\top} \xopt & = \left( \v_j^{\top}  \left(\Z\Z^{\top}\left(\frac{-1}{2 \sqrt{\nnem}} \sum_{i \in [\nnem], i \neq j}\v_i \right)\right) \right) - \frac{1}{2 \sqrt{\nnem}} \v_j^{\top} \Z\Z^{\top} \v_j \\
    \label{eqn:vjxbound}
    & =  \left((\Z\Z^{\top} \v_j)^{\top}  \left(\frac{-1}{2 \sqrt{\nnem}} \sum_{i \in [\nnem], i \neq j}\v_i \right) \right) - \frac{1}{2 \sqrt{\nnem}} \norm{\Z^{\top}\v_j}_2^2.
\end{align}
 Let $\bar{\v}_{-j} =  \left(\frac{-1}{2 \sqrt{\nnem }}\sum_{i \in [\nnem], i \neq j} \v_i \right)$. By Fact \ref{fact:sum+subgauss}, each coordinate of $\bar{\v}_{-j}$ is sub-Gaussian with sub-Gaussian parameter $1/\sqrt{4d}$. Also by  Fact \ref{fact:sum+subgauss}, $(\Z\Z^{\top} \v_j)^{\top}\bar{\v}_{-j} $ is sub-Gaussian with sub-Gaussian parameter $\norm{\Z\Z^{\top} \v_j}_2/\sqrt{4d}=\norm{\Z^{\top} \v_j}_2/\sqrt{4d}$. Now using \cref{fact:subgauss_conc}, with failure probability at most $2/d^C$,
\begin{equation}
    \label{eqn:firstbound}
    \abs{ (\Z\Z^{\top} \v_j)^{\top}\bar{\v}_{-j}  } \leq \frac{\sqrt{2C \log(d)} \norm{\Z^{\top} \v_j}_2}{\sqrt{4d}}.
\end{equation}
We now turn to bounding $ \norm{\Z^{\top} \v_j}_2$. Note that each coordinate of $\v_j$ is sub-Gaussian with sub-Gaussian norm $2/\sqrt{d}$ and $\E[\v_j \v_j^{\top}]=(1/d) \id_d$. Therefore, by Lemma \ref{lemma:subgauss_proj}, with failure probability at most $2 \exp(-c_2d)$ for some absolute constant $c_2>0$,
\begin{equation}
    \label{eqn:boundednorm}
    \frac{1}{4} \leq \frac{d-r}{d}-1/4 \leq \norm{\Z^{\top} \v_j}_2^2 \leq \frac{d-r}{d} +1/4 \leq 2
\end{equation}
where we have used the fact that $r\le d/2$.
Then if \cref{eqn:firstbound} and \cref{eqn:boundednorm} hold,
\begin{align}
\label{eqn:part1}
    \abs{ (\Z\Z^{\top} \v_j)^{\top}\bar{\v}_{-j}  } \leq \sqrt{\frac{ {C \log(d)}  }{ {d} }}.
\end{align}
Moreover, under \cref{eqn:boundednorm},
\begin{equation}
\label{eqn:part2}
    \frac{1}{2 \sqrt{\nnem}} \norm{\Z^{\top}\v_j}_2^2  \geq \frac{1}{8 \sqrt{\nnem}}.
\end{equation}
Combining \cref{eqn:vjxbound}, \cref{eqn:part1}, and \cref{eqn:part2} we find that with failure probability at most $2/d^C +2 \exp(-c_2 d)$,
\begin{equation}
    \v_j^{\top} \xopt \leq  - \frac{1}{8 \sqrt{\nnem}} + \sqrt{\frac{C \log(d)}{ {d}}}.
\end{equation}
Then, using a union bound we conclude that with failure probability at most $\nnem \left(2/d^C +2\exp(-c_2 d) \right) $,
\begin{equation}
   \max_{j \in [\nnem]} \v_j^{\top} \xopt \leq - \frac{1}{8 \sqrt{\nnem}} + \sqrt{\frac{C \log(d)}{ {d}}}.
\end{equation}
Finally since $\eta \linf{\A \x} - \rho = - \rho  = -1 \leq -1/(8 \sqrt{\nnem})$, we conclude
\begin{equation}
    F_{\nem, \A}^{(N)} (\xopt) \leq  -\frac{1}{8 \sqrt{N}} + \sqrt{\frac{C \log(d)}{ {d}}} - \gamma.
\end{equation}
We now take $C = 3$, which gives us a overall failure probability of $2 \exp(-c_1 d) + \nnem \left(2/d^C +2 \exp(-c_2 d) \right) \le 2/d$ (for sufficiently large $d$ and since $\nnem\le d)$. Noting that $\gamma >  \sqrt{3\log(d)/{d}}$ finishes the proof.
\end{proof}

The following simple lemma shows that all the Nemirovski vectors are unique with high probability, and will be useful for our definition of informative subgradients.

\begin{lemma}\label{lem:nem_unique}
 For sufficiently large $d$, with failure probability at most $1/d$, $\v_i \ne \v_j \; {\forall \;  i, j \in [\nnem]} , {i \ne j}$.
\end{lemma}
\begin{proof}
The proof follows from the birthday paradox since $\v_i$ are drawn from the uniform distribution over a support of size $2^d$: the probability of $\v_i \ne \v_j \; \forall \;  i, j \in [\nnem] , i \ne j$ is equal to,
\begin{align}
    1 \cdot \frac{2^d-1}{2^d} \cdot \frac{2^d-2}{2^d} \dots \frac{2^d-(N-1)}{2^d} \ge \left(\frac{2^d-d}{2^d}\right)^d = \left(1-\frac{d}{2^d}\right)^d.
\end{align}
For $d$ sufficiently large we can write,
\begin{align}
    \left(1-\frac{d}{2^d}\right)^d \ge e^{-2d^2/2^d} \ge 1-\frac{2d^2}{2^d}\ge 1-\frac{1}{d}
\end{align}
which completes the proof.
\end{proof}

Finally, we note that for the orthogonality condition from the definition of a \functionclassname\ is satisfied with our definition of the subgradient oracle. 

\begin{lemma}
\label{lemma:orthogonaltoA}
For any $j \in [N]$ the following is true: For any $\x \in \R^d$ such that $F_{\nem, \A}^{(j)}(\x) \neq \eta \linf{\A \x} - \rho$, either $\g_{F}^{(j)}(\x) = \v_1$ or $\linf{\A \x}/\norm{\x}_2 \leq 1/d^4$. 
\end{lemma}
\begin{proof}
Fix any $j \in [N]$ and assume $\x$ is such that  $F_{\nem, \A}^{(j)}(\x) \neq \eta \linf{\A \x} - \rho$. Note that since $F_{\nem, \A}^{(j)}(\x) \neq \eta \linf{\A \x} - \rho$, $\g_{F}^{(j)}(\x)=\v_1$, or $\g_{F}^{(j)}(\x)=\v_k$ for some $k \in \left \{ 2, \dots, j \right \}$. We consider the case when $\g_{F}^{(j)}(\x)=\v_k$ for some $k \in \left \{ 2, \dots, j \right \}$. 
This implies that $\eta \linf{\mat{A} \x} - \rho \leq \v_k^{\top} \x - k \gamma$. Observe that $\v_k^{\top} \x - k \gamma \leq \norm{\v_k}_2 \norm{\x}_2$ and $\norm{\v_k}_2 = 1$. Therefore
\begin{equation}
   \eta \linf{\mat{A} \x} - \rho \leq  \norm{\x}_2.
\end{equation}
Next we bound $\norm{\x}_2$ from below. For $k \in \left \{ 2, \dots, j \right \}$, $\g_{F}^{(j)}(\x)  = \v_k$ implies that
\begin{equation}
    \v_k^{\top} \x - k \gamma  \geq  \v_{k-1}^{\top} \x - (k-1) \gamma \implies (\v_k - \v_{k-1})^{\top} \x \geq \gamma.
\end{equation}
Therefore,
\begin{equation}
    \norm{\x}_2 \geq \frac{(\v_k - \v_{k-1})^{\top}}{\norm{\v_k - \v_{k-1}}_2} \x \geq \frac{\gamma}{\norm{\v_k - \v_{k-1}}_2} \geq \frac{\gamma}{2}.
\end{equation}
Since $\norm{\x}_2>0$, we can write
\begin{equation}
    \eta \linf{\mat{A} \x} - \rho \leq  \norm{\x}_2 \implies \frac{\linf{\mat{A} \x} }{\norm{\x}_2} \leq \frac{1}{\eta} + \frac{\rho }{\eta \norm{\x}_2} \leq \frac{1}{\eta} + \frac{2 \rho}{\eta \gamma}.
\end{equation}
Then recalling that $\eta = d^5$, $\rho = 1$ and $\gamma = \sqrt{c \nreturn \log(d)/d}$,
\begin{equation}
    \frac{\linf{\mat{A} \x}}{\norm{\x}_2}  \le \thetathreshold
\end{equation}
which completes the proof.

\end{proof}

Putting together all the previous results, we show Theorem \ref{theorem:nem_memory_sensitive} which establishes that the Nemirovski function class has the required memory-sensitive properties.
\begin{theorem}
\label{theorem:nem_memory_sensitive}
Consider any $\A\in \R^{n\times d}$ where $n\le d/2$ and any row has $\ell_2$-norm bounded by $d$. Fix $\eta = d^5$, $\rho = 1$ and $\gamma=\sqrt{\frac{c \roundlen \log(d) }{d}}$ for $c = 400$. 
Let $\nnem \leq (1/32 \gamma)^{2/3}$. 
Then $\nem_{\nnem, \gamma}$ is a $(d^6, \nnem, \roundlen, 1/(20\sqrt{\nnem}))$-\functionclassname. That is, we can sample $(f, \g_f)$ from $\nem_{\nnem, \gamma}$ with at most $d^2$ random bits, and for $(f, \g_f) \sim \nem_{\nnem, \gamma}$, with failure probability at most $5/d$, 
any algorithm for optimizing $F_{f, \A}(\x)$ which makes fewer than $d^2$ queries to oracle $\g_f$ has the following properties:
\begin{enumerate}
 \item $F_{f, \A}$ is convex and  $d^6$-Lipschitz.
\item  With $S_j$ and $\x_{t_j}$ as defined in Definition \ref{def:memorysensitive},
\begin{equation}
   \norm{\proj_{S_{j-1}}({\x}_{t_j})}_2/\norm{{\x}_{t_j}}_2 \leq 1 - 1/d^2.
\end{equation}
\item Any query $\x$ such that $F_{f, \A, \eta, \rho}(\x) \neq \eta \linf{\A \x}- \rho$ satisfies 
    \begin{equation}
        \g_{F_{f, \A,}}(\x) = \v_1 \textrm{ or } \linf{\A \x}/\norm{\x}_2\leq \thetathreshold.
    \end{equation}
\item Any algorithm which has queried $r < N$ unique gradients from $f$ has a sub-optimality gap of at least 
\begin{equation}
     F_{f,\A}(\x_{\numqueries}) - F_{f,\A}(\x \opt)\geq \frac{1}{20 \sqrt{\nnem}}.
\end{equation}
\end{enumerate}

\end{theorem}
\begin{proof}
Note that by Lemma \ref{lem:sample_nemi}, the first-order oracle $\oracle_{F_{\nem, \A}^{(N)}}(\x) = (F_{\nem, \A}^{(N)}(\x), \g_F^{(N)}(\x))$ has the same distribution over responses as $\nem_{N,\gamma, \A}$. Therefore we will analyze $\oracle_{F_{\nem, \A}^{(N)}}$, using \playerb.

We consider the following sequence of events: (a) event $E$ defined in Lemma \ref{lem:vector_condition} holds, (b) no two Nemirovski vectors are identical, $\v_i \ne \v_j$ for $i\ne j$, (c) \cref{eq:opt} regarding the value for $F_{f,\A}^{(N)}(\x \opt)$ holds. We claim that for any algorithm which makes fewer than $d^2$ queries, events (a) through (c) happen with failure probability at most $5/d$. To verify this, we do a union bound over the failure probability of each event: for (a), Lemma \ref{lem:vector_condition} shows that $E$ holds with failure probability at most $1/d$ for any  algorithm which makes fewer than $d^2$ queries, (b) holds with probability $1/d$ from Lemma \ref{lem:nem_unique} and (c) holds with probability $2/d$ from Lemma \ref{lem:subopt}. 
Now note that by Lemma \ref{lem:gradient_consistent}, \playerb's responses are identical to the responses of oracle $\oracle_{F_{\nem, \A}^{(N)}}$ under $E$, and therefore when conditioned on events (a) through (c).  
Therefore, we will condition on events (a) through (c) all holding and consider \playerb\ to prove all four parts of the theorem.

We begin with the first part. Note that $f(\x)$ has Lipschitz constant bounded by $1$. Next note that since each $\vec{a} \in B_d$ has $\norm{\vec{a}}_2 \leq d$ and $\eta = d^5$ we have that the $\left \{ \eta \linf{\A \x} - \rho \right \}$ term has Lipschitz constant bounded by $L \leq d^6$. Therefore, $F_{f, \A}$ has Lipschitz constant bounded by $d^6$.
For the second part we first note that under event (b) all Nemirovski vectors are unique, therefore the vectors $\x_{t_j}$ and subspace $S_{j-1}$ are well-defined. Now under event $E$ by applying Lemma \ref{lem:highrank},
\begin{equation}
    \norm{\proj_{S_{j-1}}(\x_{t_j})}_2/\norm{\x_{t_j}} \leq 1 - 1/d^2.
\end{equation}
The third part holds by Lemma  \ref{lemma:orthogonaltoA}. For the final part, we first note that under event (b) if \playera\ observes $r$ unique Nemirvoski vectors then it has only proceeded to the $\nth{(r+1)}$ Phase of \playerb. Now by Lemma \ref{lem:subopt}, under event $E$, any algorithm in the  $\nth{(r+1)}$ Phase of \playerb\ has function value at least  $-(r+2)\gamma$. Combining this with \cref{eq:opt} holding, for any algorithm which observes at most than $r$ unique Nemirovski vectors as gradients,
\begin{align}
     F_{f,\A}^{(N)}(\x_{\numqueries}) - F_{f,\A}^{(N)}(\x \opt) &\geq  - (r+2)\gamma + {1}/{(8 \sqrt{\nnem})} 
    \ge - (\nnem + 2) \gamma + {1}/{(8 \sqrt{\nnem})}
     \ge {1}/{(20\sqrt{\nnem})} ,
\end{align}
where the first inequality is by the above observation and Lemma~\ref{lemma:optimalvalue} and the final inequality holds by noting that $\nnem \leq (1/32 \gamma)^{2/3}$ and so $-(\nnem + 2) \gamma \geq  -1/(20 \sqrt{\nnem})$.
\end{proof}
\section{Memory Lower Bounds for the \orthgame}\label{sec:entropy}
In this section, we prove the lower bound for the \orthgame\ (Game \ref{alg:gradient_game}), building on the proof sketch in Section \ref{sec:entropy_sketch}. We prove a slightly stronger result in that we allow the oracle response ${\g}_i \in \R^{d}$ to any query $\x_i \in \R^d$ be an arbitrary (possibly randomized) function from $\R^d\rightarrow \{ \vec{a}_1, \dots, \vec{a}_n\}$.

Recall that ${\G}=\left( \vec{g}_1, \dots, \vec{g}_m \right)^{\top} \in \R^{m \times d}$.  We perform a small augmentation of the matrix ${\G}$ to streamline our analysis, {defined by the following matching between the rows of $\A $ and $\G$.} 
{ For every row of $\A$, if that row is also present in the matrix $\G$, then we match it to that row of $\G$. If some row of $\A$ is present multiple times in $\G$, then we match it to its first occurrence in $\G$. We define any matched row in $\G$ as a \emph{unique row}. Let $\muniquenew \in [m]$ denote the number of such {unique rows}. If $\muniquenew< m$, construct a matrix $\Tilde{{\G}} \in \R^{(2m-\muniquenew) \times d}$ by appending $m-\muniquenew$ additional rows of $\A$ to  ${\G}$, choosing the first $m-\muniquenew$ rows of $\A$ which were previously unmatched. We now match $\A$ to $\Tilde{{\G}}$ in the same way, and note that $\Tilde{\G}$ now has exactly $m$ unique rows by definition. Given $\A$ and $\Tilde{\G}$ let $\Ap \in \R^{(n-\munique) \times d}$ denote the matrix $\A$ when all the rows of $\A$ which were matched to some row in $\Tilde{\G}$ are removed.  
 } 
Drawing on these definitions, we begin with the following relation between the entropy of $\A$ and $\Ap$ conditioned on $\Tilde{\G}$ and $R$.

\begin{lemma}
\label{lemma:hardinfoboundprecursor}
$ \ent(\A \lvert \Tilde{\G}, R)  \leq  \ent(\Ap \lvert \Tilde{\G}, R) + 2\munique \log(4n)$.
\end{lemma}
\begin{proof}
Note that for any random variable $X$, there exists a description of the random variable with expected description length $L_X$ bounded as \cite[Chapter~5]{CoverTh91},
\begin{align}
    H(X)\le L_X \le H(X)+1.\label{eq:descrip_len}
\end{align}
Now fix any $\Tilde{\G}=\Tilde{\G}'$ and $R=R'$. Let $h_{\Tilde{\G}',R'} \defeq \ent(\Ap \lvert \Tilde{\G}=\Tilde{\G}', R=R')$. Note that by definition, 
\begin{align}
\ent(\Ap \lvert \Tilde{\G}, R) = \sum_{\Tilde{\G}'}\sum_{R'} h_{\Tilde{\G}',R'} \prob(\Tilde{\G}=\Tilde{\G}') \prob(R=R').
\end{align}
Now by \cref{eq:descrip_len} if $\ent(\Ap \lvert \Tilde{\G}=\Tilde{\G}', R=R')=h_{\Tilde{\G}',R'}$, then given $\Tilde{\G}=\Tilde{\G}'$ and $R=R'$  there exists some description of $\Ap$ which has expected description length at most $h_{\Tilde{\G}',R'}+1$. Using this we construct a description of $\A$ given $\Tilde{\G}=\tilde{\G}'$ and $R=R'$ which has expected description length at most $h_{\Tilde{\G}',R'}+1 + 2\munique\log(n)$. To do this, note that if there is already a description of $\Ap$ and we are given $\Tilde{\G}=\Tilde{\G}'$, then to specify $\A$ we only need to additionally specify the $\munique$ rows which were removed from $\A$ to construct $\Ap$. For every row which was removed from $\A$, we can specify its row index in the original matrix $\A$ using $\lceil \log(n) \rceil\le \log(2n)$ bits, and specify which of the rows of $\Tilde{\G}'$ it is using another $\lceil \log(2m-\muniquenew) \rceil\le \lceil \log(2m) \rceil \le\log(4n)$ bits. Therefore, given $\Tilde{\G}=\tilde{\G}'$ and $R=R'$,
\begin{align}
\ent(\A \lvert \Tilde{\G}=\Tilde{\G}', R=R') &\le h_{\Tilde{\G}',R'}+1 + \munique\log(2n)+m\log(4n)\le h_{\Tilde{\G}',R'} + 2\munique\log(4n) \\
\implies \ent(\A \lvert \Tilde{\G}, R) &= \sum_{\Tilde{\G}'}\sum_{R'} \ent(\A \lvert \Tilde{\G}=\Tilde{\G}', R=R')  \prob(\Tilde{\G}=\Tilde{\G}') \prob(R=R') \\
&\le \ent(\Ap \lvert \Tilde{\G}, R)+2\munique\log(4n).
\end{align}
\end{proof}
By a similar analysis we can bound the entropy of $\Tilde{\G}$.

\begin{lemma}
\label{lemma:Gentropy}
$H(\Tilde{\G})\le \munique \log(\abs{B_d}) + 2m\log(2n)$.
\end{lemma}
\begin{proof} We claim that $\Tilde{\G}$ can be written with at most $\munique \log(\abs{B_d}) + 2m\log(2n)$ bits, therefore its entropy is at most $\munique\log(\abs{B_d}) + 2m\log(2n)$ (by \cref{eq:descrip_len}). To prove this, note that $\Tilde{\G}$ consists of rows of $\A$ which are sampled from $B_d$, and recall that $\tilde{\G}$ has exactly $m$ unique rows. Therefore we can first write down all the $\munique$ unique rows of $\Tilde{\G}$ using  $\munique \log(\abs{B_d})$ bits, and then for each of the at most $(2m-\muniquenew)\le 2m$ rows of $\Tilde{\G}$ we can write down which of the unique rows  it is with $\lceil \log(\munique)\rceil \le \log(2n)$ bits each, for a total of at most $\munique \log(\abs{B_d}) + 2m\log(2n)$ bits.
\end{proof}

The following lemma shows that $Y$ is a deterministic function of $\tilde{\G}, \messages$ and $R$.

\begin{lemma}
\label{lemma:remove_x}
The queries $\{\x_i, i \in [m]\}$ are a deterministic function of $\G, \messages$ and $R$. Therefore $\Y=g(\tilde{\G},\messages,R)$ for some function $g$.
\end{lemma}
\begin{proof}
Given $\A$ and $R$ we observe that \playerone\ is constrained to using a deterministic algorithm to both construct message $\messages$ to store and to determine the queries $(\x_1, \dots, \x_m)$. Therefore we see that for any $i$ there exists a function $\phi_i$ such that for any $\messages$, $R$, and responses $(\g_1, \dots, \g_{i-1})$, $\x_i = \phi_i(\messages, R, \x_1, \g_1, \dots, \x_{i-1}, \g_{i-1})$.
Next we remark that there exists a function $\phi_i'$ such that $\x_i = \phi_i'(\messages, R,  \g_1, \dots, \g_{i-1})$.  We can prove this by induction. For the base case we simply note that $\phi_1' = \phi_1$ and $\x_1 = \phi_1'(\messages, R)$. Next assume the inductive hypothesis that for any $j \leq i-1$ we have $\x_j = \phi_j'(\messages, R, \g_1, \dots, \g_{j-1})$. Then
\begin{align}
    \x_i & = \phi_i(\messages, R, \x_1, \g_1, \dots, \x_{i-1}, \g_{i-1} ) \\
    & = \phi_i(\messages, R, \phi'_1(M, R), \g_1, \dots, \phi'_{i-1}(\messages, R, \g_1, \dots, \g_{i-2}), \g_{i-1} ).
\end{align}
Thus we define
\begin{align}
    \phi'_i (\messages, R, \g_1, \dots, \g_{i-1}) \defeq \; \phi_i\Big(\messages, R, \phi'_1(M, R), \g_1, 
    \dots, \phi'_{i-1}(\messages, R, \g_1, \dots, \g_{i-2}), \g_{i-1} \Big).
\end{align}
Therefore given $(\messages, R, \G)$, it is possible to reconstruct the queries $\X$. Since $\Y$ is a deterministic function of $(\X,\G,\messages,R)$ in the \orthgame, and $\G$ is just the first $m$ rows of $\tilde{\G},$ $\Y=g(\tilde{\G},\messages,R)$ for some function $g$.
\end{proof}

As sketched out in Section \ref{sec:entropy_sketch}, in the following two lemmas we compute $I(\Ap; \mat{Y} \lvert \Tilde{\G}, R)$ in two different ways.

\begin{lemma}
\label{lemma:upperbound}
$I(\Ap; \mat{Y} \lvert \Tilde{\G}, R) \leq \memsize$.
\end{lemma}
\begin{proof} We note that by Lemma \ref{lemma:remove_x}, $\mat{Y} = g(\Tilde{\G}, \messages, R)$  and therefore by the data processing inequality,
\begin{equation}
    I(\Ap ; \mat{Y} \lvert \Tilde{\G}, R) \leq I(\Ap ; \Tilde{\G},\messages, R \lvert \Tilde{\G}, R)=I(\Ap ; \messages \lvert \Tilde{\G}, R).
\end{equation}
We conclude by noting that 
\begin{equation}
    I(\Ap; \messages \lvert \Tilde{\G}, R) \leq \ent(\messages \lvert \Tilde{\G}, R) \leq \ent(\messages) \leq \memsize
\end{equation}
where in the last step we use the fact that $\messages$ is $M$-bit long.
\end{proof}

\begin{lemma}
\label{lemma:lowerbound}
Suppose $\mat{Y}$ wins the \twoorthgame\ with failure probability at most $\delta$. Then if $\left \{ B_d \right \}_{d = 1}^{\infty}$ is a memory sensitive base with constant $\msbconstant >0 $ ,
\begin{equation}
    I(\Ap; \mat{Y} \lvert \Tilde{\G}, R ) \geq  (n - \mtotal )(1-\delta) \frac{\msbconstant}{2} \nreturn - 4 \mtotal \log(4n).
\end{equation}
\end{lemma}

\begin{proof}
We have
\begin{equation}
    \label{eqn:infobreakdown}
    I(\Ap; \mat{Y} \lvert \Tilde{\G}, R ) = \ent(\Ap \lvert \Tilde{\G}, R ) - \ent(\Ap \lvert \mat{Y} ,\Tilde{\G}, R).
\end{equation}
We will lower bound $I(\Ap; \mat{Y} \lvert \Tilde{\G}, R )$  by providing a lower bound for $\ent(\Ap \lvert \Tilde{\G}, R)$ and an upper bound for $\ent(\Ap \lvert \mat{Y}, \Tilde{\G}, R)$. To that end, by Lemma \ref{lemma:hardinfoboundprecursor} we have
\begin{equation}
    \ent(\Ap \lvert \Tilde{\G}, R) \geq \ent(\A \lvert \Tilde{\G}, R) - 2\munique \log (4n). 
\end{equation}
We can lower bound $\ent(\A \lvert \Tilde{\G}, R)$ as follows,
\begin{align}
    \ent(\A \lvert \Tilde{\G}, R) &= \ent(\A \lvert R) - I(\A ; \Tilde{\G} \lvert R) 
    = \ent(\A \lvert R) - \left( \ent(\Tilde{\G} \lvert R) - \ent(\Tilde{\G} \lvert \A, R) \right) \\
    &\geq \ent(\A \lvert R) - \ent(\Tilde{\G} \lvert R)
    \geq \ent(\A \lvert R) - \ent(\Tilde{\G})
\end{align}
where the inequalities use the fact that entropy is non-negative, and that conditioning reduces entropy. We now note that since $\A$ is independent of $R$, $\ent(\A \lvert R)=\ent(\A) $. Thus we have
\begin{equation}
    \ent(\Ap \lvert \Tilde{\G}, R) \geq H(\A) - H(\Tilde{\G}) - 2\munique \log(4n),
\end{equation}
and so recalling Lemma \ref{lemma:Gentropy} and that $H(\A) = n \log(\abs{B_d})$ we conclude
\begin{equation}
\label{eqn:entbound1}
    \ent(\Ap \lvert \Tilde{\G}, R) \geq   n \log(\abs{B_d}) \ -   \munique \log(\abs{B_d}) - 4\munique \log(4n). 
\end{equation}
Next we upper bound $\ent(\Ap \lvert \mat{Y}, \Tilde{\G}, R)$. First we note
\begin{equation}
\label{eq:entropy_ap}
    \ent(\Ap \lvert \mat{Y}, \Tilde{\G}, R) \leq \ent(\Ap \lvert \mat{Y}) \leq \sum_{j=1}^{n-m} H(\vec{a}'_j \lvert \mat{Y})  
\end{equation}
where $\vec{a}'_j$ is the $\nth{j}$ row of $\Ap$. 
Next recall that if $\mat{Y} = (\y_1, \dots, \y_{\nreturn})^{\top}  \in \R^{\nreturn \times d}$ wins the \twoorthgame\ then for any $\y_i$, $\linf{\A \y_i}/\norm{\y_i}_2 \leq \thetathreshold$ and for $S_i = \spcommand \left( \y_1, \dots, \y_{i-1} \right)$, we have $\norm{\proj_{S_i}(\y_i)}_2/\norm{\y_i}_2 \leq 1 - (1/d^{\projpower})$. Since these properties are normalized by $\norm{\y_i}_2$ we may, without loss of generality, overload our notation for $\y_i$ by normalizing, we set $\y_i \gets \y_i/\norm{\y_i}_2$. Note that for any $j$ the entropy of $\vec{a}'_j$ conditioned on any value $\mat{Y}'$ taken by the random variable $\mat{Y}$ is bounded above by the entropy of $\vec{a}'$ where the law of $\vec{a}'$ corresponds to the uniform distribution over the set $ \left \{ \vec{a} \in B_d \textrm{ s.t. } \linf{\mat{Y}' \vec{a}} \leq \frac{1}{d^4} \right \}$ and this has entropy $\log \left( \abs{\left \{ \vec{a} \in B_d \textrm{ s.t. } \linf{\mat{Y}' \vec{a}} \leq \frac{1}{d^4} \right \}} \right)$. We also note that since $\mat{Y} = g(\Tilde{\G}, \messages, R)$ (Lemma \ref{lemma:remove_x}), and all of $\tilde{\G}, \messages$ and $R$ take values on a finite support, $\Y$ also takes values on some finite support ($\Y$ can still be real-valued, but its support must be finite). Therefore, for any $j$, we can write,
\begin{align}  
    H(\vec{a}'_j \lvert \mat{Y}) & =  \sum_{\mat{Y}'} \left[ H \left( \vec{a}'_j\lvert \mat{Y} = \mat{Y}' \right) \right]\prob(\Y=\Y')\\
    & =  \sum_{\mat{Y}': \mat{Y}' \textrm{ wins} } \left[ H \left(\vec{a}'_j \lvert \mat{Y} = \mat{Y}' \right) \right] \prob(\Y=\Y') +  \sum_{\mat{Y}': \mat{Y}' \textrm { loses } } \left[ H \left( \vec{a}'_j \lvert \mat{Y} = \mat{Y}' \right) \right] \prob(\Y=\Y') \\
    & \leq \sum_{\mat{Y}': \mat{Y}' \textrm{ wins} } \left[ H \left( \vec{a}'_j \lvert \mat{Y} = \mat{Y}' \right) \right]\prob(\Y=\Y')+  p(\mat{Y} \textrm{ loses}) H(\vec{a}'_j) \\
    & \leq \sum_{\mat{Y}': \mat{Y}' \textrm{ wins} }   \log \left( \abs{\left \{ \vec{a} \in B_d \textrm{ s.t. } \linf{\mat{Y}' \vec{a}} \leq \frac{1}{d^4} \right \}} \right) \prob(\Y=\Y') + p(\mat{Y} \textrm{ loses}) H(\vec{a}'_j) \\
    & \leq (1 - \delta) \log\left( \max_{\mat{Y}': \mat{Y}' \textrm{wins}} \abs{ \left \{ \vec{a} \in B_d \textrm{ s.t. } \linf{\mat{Y}' \vec{a}} \leq \frac{1}{d^4} \right \}} \right) + \delta  \log(\abs{B_d}).  \label{eqn:entropytosetsize}
\end{align}

Thus for any $\mat{Y} = \left( \y_1, \dots, \y_{\nreturn} \right)$ which wins the \twoorthgame, we will upper bound  $\abs{ \left \{ \vec{a} \in B_d \textrm{ s.t. } \linf{\mat{Y} \vec{a}} \leq 1/d^4 \right \}}$. We will use the following lemma (proved in Appendix~\ref{subsec:robut_linear}) which constructs a partial orthonormal basis from a set of robustly linearly independent vectors. 

\begin{restatable}{lemma}{algebrahelper}
\label{lemma:algebrahelper}
Let $\lambda \in (0, 1]$ and suppose we have $\ntot \leq d$ unit norm vectors $\y_1, \dots, \y_{\ntot} \in \R^d$. Let $S_i \defeq \spcommand \left( \left \{ \y_1, \dots, \y_{i} \right \} \right), S_0 \defeq \phi$. Suppose that for any $i \in [\ntot]$, 
\begin{equation}
    \norm{\proj_{S_{i-1}}(\y_i)}_2 \leq 1 - \lambda.
\end{equation}
Let $\mat{Y} = \left( \y_1, \dots, \y_{\ntot} \right) \in \R^{d \times \ntot}$. There exists $\nrem = \lfloor \ntot/2 \rfloor$ orthonormal vectors $\vec{m}_1, \dots, \vec{m}_{\nrem}$ such that, letting $\mat{M} = \left( \vec{m}_1, \dots, \vec{m}_{\nrem} \right)$, for any $\vec{a} \in \R^d$, 
\begin{equation}
    \linf{\mat{M}^{\top} \vec{a}} \leq \frac{d}{\lambda} \linf{\mat{Y}^{\top} \vec{a}}.
\end{equation}
\end{restatable}

By Lemma \ref{lemma:algebrahelper} there exist $\lfloor \nreturn/2 \rfloor$ orthonormal vectors $\left \{ \vec{z}_1, \dots, \vec{z}_{\lfloor \nreturn/2 \rfloor} \right \}$ such that for $\mat{Z} = \left( \vec{z}_1, \dots, \vec{z}_{\lfloor \nreturn/2 \rfloor} \right)$ and for any $\vec{a} \in \R^d$,
\begin{equation}
\label{eqn:previous}
    \linf{\mat{Z}^{\top} \vec{a}} \leq d^3 \linf{\mat{Y}^{\top} \vec{a}}.
\end{equation}
Define $S \defeq \left \{ \vec{a} \in B_d \textrm{ s.t. } \linf{\mat{Z} \vec{a}} \leq 1/d  \right \}$. By \cref{eqn:previous},
\begin{equation}
    \abs{ \left \{ \vec{a} \in B_d \textrm{ s.t. } \linf{\mat{Y}^{\top} \vec{a}} \leq \frac{1}{d^4} \right \}} \leq \abs{S}.
\end{equation}
Observing that $B_d$ is a memory-sensitive base as per Definition \ref{def:memorysensitivebase},
\begin{equation}
    \abs{S} \leq  \mathbb{P}_{\vec{a} \sim \mathsf{Unif}(B_d)} \left( \linf{\mat{Z}^{\top} \vec{a}} \leq 1/d \right) \abs{ B_d } \leq 2^{- \msbconstant \nreturn/2} \abs{B_d},
\end{equation}
for some constant $\msbconstant > 0$. From \cref{eqn:entropytosetsize},
\begin{equation}
     H(\vec{a}'_j \lvert \mat{Y}) \leq (1 - \delta) \log \left( 2^{- \msbconstant \nreturn /2} \abs{ B_d } \right) + \delta \log \left( \abs{ B_d}  \right) =   \log( \abs{B_d} )  - (1- \delta)  \msbconstant \nreturn /2.
\end{equation}
Plugging this into \cref{eq:entropy_ap},
\begin{equation}
    H(\Ap \lvert \mat{Y}) \leq (n-\munique ) \left(  \log( \abs{B_d} )  - (1- \delta)  \msbconstant \nreturn /2 \right).
\end{equation}
Recalling \cref{eqn:infobreakdown} and \cref{eqn:entbound1} we conclude,
\begin{align}
    I(\Ap; \mat{Y} \lvert \Tilde{\G}, R) & \geq   \log(\abs{B_d}) (n - \munique ) - 4\munique \log(4n)  - (n-\munique ) \left(  \log( \abs{B_d} )  - (1- \delta)  \msbconstant \nreturn/2  \right) \\
    & = (n -\munique )(1-\delta) \frac{\msbconstant}{2} \nreturn - 4\munique   \log(4n).
\end{align}
\end{proof}

By combining Lemma~\ref{lemma:upperbound} and \ref{lemma:lowerbound}, we prove the memory-query tradeoff for the \orthgame.

\gamebound*
\begin{proof}

By Lemma \ref{lemma:upperbound} we have that
\begin{equation}
    I(\Ap; \vec{Y} \lvert \Tilde{\G}, R) \leq \memsize.
\end{equation}
However if the player returns some $\mat{Y}$ which wins the \orthgame\ with failure probability at most $\delta$ then by Lemma \ref{lemma:lowerbound},
\begin{equation}
    I(\Ap; \vec{Y} \lvert \Tilde{\G}, R) \geq (n - \mtotal ) (1-\delta)\frac{\msbconstant}{2} \nreturn - 4 \mtotal \log(4n).
\end{equation}
Therefore we must have
\begin{equation}
    (n - \mtotal ) (1-\delta)\frac{\msbconstant}{2} \nreturn - 4 \mtotal \log(4n) \leq M.
\end{equation}
Rearranging terms gives us
\begin{equation}
    \mtotal \geq \frac{n (1-\delta) \msbconstant \nreturn/2 -  \memsize}{(1-\delta)\msbconstant \nreturn/2 + 4 \log(4n)}.
\end{equation}
Recall that $k = \lceil 60 M / (\msbconstant d) \rceil$ and thus if $M \geq d \log(4n)$ then for $\delta = 2/3$ we have $ (1 - \delta) \msbconstant k/2 \geq 4 \log(4n)$. Using this and the fact that $n = \lfloor d/2 \rfloor$,
\begin{align}
    \mtotal \geq \frac{n (1-\delta) \msbconstant \nreturn/2 -  \memsize}{(1-\delta)\msbconstant \nreturn}
      \geq \frac{n}{2} -\frac{M}{(1-\delta)\msbconstant \nreturn}
      \geq \frac{n}{2}-\frac{d}{20}
      \geq d/5.
\end{align}
\end{proof}

\section*{Acknowledgements}

Thank you to anonymous reviewers for helpful feedback on earlier drafts of this paper. Annie Marsden and Gregory Valiant were supported by NSF Awards AF-2341890, CCF-1704417, CCF-1813049 and a Simons Foundation Investigator Award. Annie Marsden was also supported by J. Duchi's Office of Naval Research Award YIP N00014-19-2288 and NSF Award HDR 1934578. Vatsal Sharan was supported by NSF CAREER Award CCF-2239265 and an Amazon Research Award. Aaron Sidford was supported by a Microsoft Research Faculty Fellowship, NSF CAREER Award CCF-1844855, NSF Grant CCF-1955039, a PayPal research award, and a Sloan Research Fellowship.

\bibliographystyle{alpha}

\bibliography{refs}

\appendix

\section{Additional Helper Lemmas}\label{app:helper}

\subsection{Useful concentration bounds}\label{sec:concentration}
In this section we establish some concentration bounds which we repeatedly use in our analysis.

\begin{definition}[sub-Gaussian random variable]
A zero-mean random variable $X$ is \emph{sub-Gaussian} if for some constant $\sigma$ and for all $\lambda \in \R$, $\E[e^{\lambda X}]\le e^{\lambda^2 \sigma^2/2}$. We refer to $\sigma$ as the \emph{sub-Gaussian parameter of $X$}. We also define the \emph{sub-Gaussian norm} $\norm{X}_{\psi_2}$ of a sub-Gaussian random variable $X$ as follows,
\begin{equation}
    \norm{X}_{\psi_2} \defeq \inf \left \{ K > 0 \textrm{ such that } \E[\exp(X^2/K^2)] \leq 2 \right \}.
\end{equation}
\end{definition}

We use that $X\sim \mathsf{Unif}(\{\pm 1\})$ is sub-Gaussian with parameter $\sigma=1$ and sub-Gaussian norm $\norm{X}_{\psi_2} \le 2$. The following fact about sums of independent sub-Gaussian random variables is useful in our analysis.
\begin{fact}
\label{fact:sum+subgauss}\citep{vershynin2018high}
For any $n$, if $\{X_i, i \in [n]\}$ are independent zero-mean sub-Gaussian random variables with sub-Gaussian parameter $\sigma_i$, then $X'=\sum_{i \in [n]} X_i$ is a zero-mean sub-Gaussian with parameter $\sqrt{\sum_{i \in [n]} \sigma_i^2}$ and norm bounded as $\norm{X'}_{\psi_2}\le  \subgaussconst\sqrt{\sum_{i \in [n]} \norm{X_i}_{\psi_2}^2}$, for some universal constant $\subgaussconst$.
\end{fact}
Using \cref{fact:sum+subgauss} we obtain the following corollary which we use directly in our analysis.
\begin{corollary}
\label{cor:xtopvgaussian}
For $\v \sim  (1/\sqrt{d}) \mathsf{Unif} \left( \mathcal{H}_d \right)$ and any fixed vector $\x \in \R^d$, $\x^{\top} \v$ is a zero-mean sub-Gaussian with sub-Gaussian parameter at most  $\norm{\x}_2/\sqrt{d}$ and sub-Gaussian norm at most $\norm{\x^{\top} \v}_{\psi_2} \leq 2\alpha \norm{\x}_2/\sqrt{d}$. 
\end{corollary}
The following result bounds the tail probability of sub-Gaussian random variables.
\begin{fact}
\label{fact:subgauss_conc}\citep{wainwright2019high}
For a random variable $X$ which is zero-mean and sub-Gaussian with parameter $\sigma$, then for any $t$, $\prob(\abs{X} \geq t ) \leq 2 \exp(-t^2/(2\sigma^2))$. 
\end{fact}

We will use the following concentration bound for quadratic forms.

\begin{theorem}[Hanson-Wright inequality \citep{hanson1971bound,rudelson2013hanson}]
\label{thm:hansonwright}
Let $\x = (X_1, \dots, X_d) \in \R^d$ be a random vector with i.i.d components $X_i$ which satisfy $\E[X_i] = 0$ and $\norm{X_i}_{\psi_2} \leq K$ and let $\mat{M} \in \R^{n \times n}$. Then for some absolute constant $\chw>0$ and for every $t \geq 0$ , 
\begin{gather}
    \max \left \{ \mathbb{P} \left(   \x^{\top} \mat{M} \x - \E \left[ \x^{\top} \mat{M} \x \right]  > t \right), \mathbb{P} \left(   \E \left[ \x^{\top} \mat{M} \x \right] - \x^{\top} \mat{M} \x  > t \right) \right \} \\\leq  \exp \left( - \chw \min \left \{ \frac{t^2}{K^4 \norm{\mat{M}}_{\textrm{F}}^2}, \frac {t}{K^2 \norm{\mat{M}}_2} \right \} \right).
\end{gather}
where $\norm{\mat{M}}_2$ is the operator norm of $\mat{M}$ and $\norm{\mat{M}}_{\textrm{F}}$ is the Frobenius norm of $\mat{M}$.
\end{theorem}

The following lemma will let us analyze projections of sub-Gaussian random variables (getting concentration bounds which look like projections of Gaussian random variables).

\begin{lemma}
\label{lemma:subgauss_proj}
Let $\x = (X_1, \dots, X_d) \in \R^d$ be a random vector with i.i.d sub-Gaussian components $X_i$ which satisfy $\E[X_i] = 0$, $\norm{X_i}_{\psi_2} \leq K$, and $\E[\x\x^{\top}]=s^2 \id_d$. Let $\Z\in \R^{d \times r}$ be a matrix with orthonormal columns $(\z_1,\dots, \z_r)$. 
There exists some absolute constant $\chw > 0$ (which comes from the Hanson-Wright inequality) such that for $t\ge 0 $,
\begin{equation}
\label{eq:subgass_proj}
   \max \left \{ \mathbb{P} \left(   \norm{\mat{Z}^{\top} \x}_2^2 - rs^2  > t \right), \mathbb{P} \left( rs^2 -  \norm{\mat{Z}^{\top} \x}_2^2  > t \right) \right \} \leq  \exp \left( - \chw \min \left \{ \frac{t^2}{r K^4 }, \frac{t}{K^2} \right \} \right).
\end{equation}
\end{lemma}

\begin{proof}
We will use the Hanson-Wright inequality (see Theorem~\ref{thm:hansonwright}). We begin by computing $\E \left[ \norm{\mat{Z}^{\top} \x}_2^2 \right]$,
\begin{align}
    \E\left[ \norm{\mat{Z}^{\top} \x}_2^2 \right] 
    & = \E \left[ \x^{\top} \mat{Z} \mat{Z}^{\top} \x \right] 
    =  \E \left[ \trace \left( \mat{Z}^{\top} \x \x^{\top} \mat{Z} \right)  \right]  
    = \trace \left( \mat{Z}^{\top}  \E \left[  \x \x^{\top} \right] \mat{Z} \right)   \\
    & = s^2 \trace \left( \mat{Z}^{\top} \mat{Z} \right) 
    = s^2 \trace( \id_{r \times r} )
     = r s^2 .
\end{align}

Next we bound the operator norm and Frobenius-norm of $\mat{Z}\mat{Z}^{\top}$. The operator norm of $\mat{Z}\mat{Z}^{\top}$ is bounded by $1$ since the columns of $\mat{Z}$ are orthonormal vectors. For the Frobenius-norm of $\mat{Z}\mat{Z}^{\top}$, note that
\begin{align}
    \norm{\mat{Z} \mat{Z}^{\top}}_{\textrm{F}}^2 & = \trace \left( \mat{Z} \mat{Z}^{\top} \mat{Z} \mat{Z}^{\top} \right) 
     = \trace \left( \mat{Z} \mat{Z}^{\top} \right) 
    = \sum_{i  \in [r]}\trace \left( \vec{z}_i \vec{z}_i^{\top} \right) 
    = \sum_{i \in [r]} \norm{\vec{z}_i}_2^2
    = r.
\end{align}
Applying the Hanson-Wright inequality (Theorem \ref{thm:hansonwright}) with these bounds on $\E \left[ \norm{\mat{Z}^{\top} \x}_2^2 \right]$, $\norm{\mat{Z} \mat{Z}^{\top}}_{2}^2$, and $\norm{\mat{Z} \mat{Z}^{\top}}_{\textrm{F}}^2$ yields \eqref{eq:subgass_proj} and completes the proof.
\end{proof}

\begin{corollary}
\label{cor:explicitprojectionbound}
Let $\v \sim  (1/\sqrt{d}) \mathsf{Unif} \left( \mathcal{H}_d \right)$ and let $S$ be a fixed $r$-dimensional subspace of $\R^d$ (independent of $\v$). Then with probability at least $1-2 \exp(\chw)/d^5$,
\begin{equation}
    \norm{ \proj_{S}(\v) }_2 \leq \sqrt{\frac{30 \log(d) r}{d}}.
\end{equation}
\end{corollary}

\begin{proof}
Let $\z_1, \dots, \z_r$ be an orthonormal basis for subspace $S$. Note that 
\begin{equation}
    \norm{\proj_S(\v)}_2^2 = \norm{\Z^{\top} \v}_2^2. 
\end{equation}
Next, for $\v = (\v_1, \dots, \v_d)$ we have $\E[\v \v^{\top}] = (1/d) \id$ and  $\norm{\v_i}_{\psi_2} \leq 2/\sqrt{d}$, thus we may set $s = 1/\sqrt{d}$ and $K = 2/\sqrt{d}$ and apply Lemma \ref{lemma:subgauss_proj},
\begin{equation}
    \mathbb{P} \left(  \abs{ \norm{ \proj_{S}(\v) }_2^2 - \frac{r}{d}} > t \right) \leq 2 \exp\left( -\chw  \min \left \{ \frac{d^2 t^2}{16 r} , \frac{dt}{4} \right \} \right).
\end{equation}
Picking $t = 20 \log(d)r/d$ concludes the proof. 
\end{proof}

\subsection{A property of robustly linearly independent vectors: Proof of Lemma \ref{lemma:algebrahelper}}\label{subsec:robut_linear}


\noindent \textsc{Lemma 5.6}. \emph{ Let $\lambda \in (0, 1]$ and suppose we have $\ntot \leq d$ unit norm vectors $\y_1, \dots, \y_{\ntot} \in \R^d$. Let $S_i \defeq \spcommand \left( \left \{ \y_1, \dots, \y_{i} \right \} \right), S_0 \defeq \phi$. Suppose that for any $i \in [\ntot]$, 
\begin{equation}
    \norm{\proj_{S_{i-1}}(\y_i)}_2 \leq 1 - \lambda.
\end{equation}
Let $\mat{Y} = \left( \y_1, \dots, \y_{\ntot} \right) \in \R^{d \times \ntot}$. There exists $\nrem = \lfloor \ntot/2 \rfloor$ orthonormal vectors $\vec{m}_1, \dots, \vec{m}_{\nrem}$ such that, letting $\mat{M} = \left( \vec{m}_1, \dots, \vec{m}_{\nrem} \right)$, for any $\vec{a} \in \R^d$, 
\begin{equation}
    \linf{\mat{M}^{\top} \vec{a}} \leq \frac{d}{\lambda} \linf{\mat{Y}^{\top} \vec{a}}.
\end{equation}
}

\begin{proof}
Since for any $i \in [\ntot]$, $\norm{\proj_{S_{i-1}}(\y_i)}_2 \leq 1 - \lambda < 1$, 
\begin{equation}
    \mathsf{dim} \hspace{1mm} \spcommand (\y_1, \dots, \y_{\ntot} ) = \ntot.
\end{equation}
 Therefore we may construct an orthonormal basis $\vec{b}_1, \dots, \vec{b}_{\ntot}$ via the Gram–Schmidt process, 
\begin{equation}
    \vec{b}_i \defeq \frac{\y_i - \proj_{S_{i-1}}(\y_i) }{ \norm{\y_i - \proj_{S_{i-1}}(\y_i) }_2}.
\end{equation}
Let $\mat{B} = \left( \vec{b}_1, \dots, \vec{b}_{\ntot} \right)\in \R^{d \times \ntot}$ and note that with this construction there exists a vector of coefficients $\vec{c}^{(i)} \in \R^{\ntot}$ such that $$\y_i = \sum_{j \in [i]} \vec{c}^{(i)}_j \vec{b}_j = \mat{B} \vec{c}^{(i)}.$$  Let $\mat{C} = \left( \vec{c}^{(1)}, \dots, \vec{c}^{(\ntot)} \right) \in \R^{\ntot\times \ntot}$ and observe that $\mat{Y} = \mat{B} \mat{C}$. Denote the singular values of $\mat{C}$ as $\sigma_1 \geq \dots \geq \sigma_{\ntot}$. We aim to bound the singular values of $\mat{C}$ from below. To this end, observe that $\mat{C}$ is an upper triangular matrix such that for any $i \in [\ntot]$, $\abs{\mat{C}_{ii}} \geq \sqrt{\lambda}$. Indeed, note that 
\begin{align}
    \mat{C}_{ii}^2 & = \norm{\y_i - \proj_{S_{i-1}}(\y_i)}_2^2  = \norm{\y_i}_2^2 - \norm{\proj_{S_{i-1}}(\y_i)}_2^2  \geq 1 - \left( 1 - \lambda \right)^2 \geq \lambda.
\end{align} 
Therefore, 
\begin{equation}
   \prod_{i \in [\ntot]} \sigma_i = \det(\mat{C}) \geq \lambda^{\ntot /2} 
\end{equation} 
where in the last step we use the fact that the determinant of a triangular matrix is the product of its diagonal entries. Next we consider the singular value decomposition of $\mat{C}$: let $\mat{U}\in \R^{\ntot\times \ntot}, \mat{V}\in \R^{\ntot\times \ntot}$ be orthogonal matrices and $\mat{\Sigma}\in \R^{\ntot\times \ntot}$ be the diagonal matrix such that $\mat{C} = \mat{U} \mat{\Sigma} \mat{V}^{\top}$. For the remainder of the proof assume without loss of generality that the columns of $\mat{Y}$ and $\mat{B}$ are ordered such that $\Sigma_{ii} = \sigma_i$. 
We will upper bound the singular values of $\mat{C}$ as follows: observe that since $\norm{\y_i}_2 = 1$, for any vector $\vec{w}$ such that $\norm{\vec{w}}_2 \leq 1$, 
\begin{equation}
    \norm{\mat{Y}^{\top} \vec{w}}_2 \le \sqrt{d},  \qquad \textrm{ and therefore } \qquad
    \norm{\mat{C}^{\top} \mat{B}^{\top} \vec{w}}_2 \leq \sqrt{d}. 
\end{equation}
In particular consider $\vec{w} = \mat{B} \vec{u}_i$ where $\vec{u}_i$ denotes the $\nth{i}$ column of $\mat{U}$. Since $\mat{B}^{\top} \mat{B} = \id$ and $\mat{U}, \mat{V}$ are orthogonal we have
\begin{align}
    \norm{\mat{C}^{\top} \mat{B}^{\top} \vec{w} }_2  = \norm{\mat{V} \mat{\Sigma} \vec{e}_i}_2 = \sigma_i \leq \sqrt{d}. 
\end{align}
Note that for $m \in [\ntot]$, if $i > \ntot - m$ then $\sigma_i \leq \sigma_{\ntot - m}$ and therefore 
\begin{align}
   \lambda^{\ntot/2} \leq \prod_{i \in [\ntot]} \sigma_i \leq \sigma_{\ntot - m}^m \prod_{i \in [\ntot-m]} \sigma_i \leq  \sigma_{\ntot - m}^m \left( \sqrt{d} \right)^{\ntot - m } \implies \sigma_{\ntot - m} \geq   \left( \lambda^{\ntot} d^{(m-\ntot)} \right)^{\frac{1}{2m}}.
\end{align}
In particular when $m =  \nrem = \lceil \ntot/2 \rceil $ we have $\sigma_{\ntot - \nrem} \geq \lambda/\sqrt{d}$. Therefore for any $i \leq \ntot - \nrem = \lfloor \frac{\ntot}{2} \rfloor$,
\begin{equation}
    \sigma_i \geq \lambda/\sqrt{d}. 
\end{equation}
Recall that $\u_i$ denotes the $\nth{i}$ column of $\mat{U}$ and let $\v_i$ denote the $\nth{i}$ column of $\mat{V}$. Observe that
\begin{equation}
     \mat{U} \mat{\Sigma} = \mat{B}^{\top} \mat{Y} \mat{V} \qquad \implies \qquad \u_i = \frac{1}{\sigma_i} \mat{B}^{\top} \mat{Y} \v_i.
\end{equation}
Extend the basis $\left \{ \vec{b}_1, \dots, \vec{b}_{\ntot} \right \}$ to $\R^d$ and let $\tilde{\mat{B}} = \left( \mat{B}, \vec{b}_{\ntot + 1}, \dots, \vec{b}_{d} \right)$ be the $d \times d$ matrix corresponding to this orthonormal basis. Note that if $j > \ntot$, then for any $\y_i$, $\vec{b}_j^{\top} \y_i = 0$. Define
\begin{equation}
    \tilde{\u}_i \defeq \frac{1}{\sigma_i} \tilde{\mat{B}}^{\top} \mat{Y} \v_i \in \R^d.
\end{equation}
For $i \in [\nrem]$ define $\vec{m}_i = \tilde{\mat{B}} \tilde{\u}_i$  and let $\mat{M} = \left( \vec{m}_1, \dots, \vec{m}_{\nrem} \right)$. Note that $\left \{ \vec{m}_1, \dots, \vec{m}_{\nrem} \right \}$ is an orthonormal set. The result then follows since for any $\vec{a} \in \R^d$,
\begin{equation}
    \abs{\vec{m}_i^{\top} \vec{a}} = \abs{\tilde{\u}_i^{\top} \tilde{\mat{B}}^{\top} \vec{a}} = \abs{ \frac{1}{\sigma_i}  \v_i^{\top} \mat{Y}^{\top} \tilde{\mat{B}} \tilde{\mat{B}}^{\top} \vec{a} } = \frac{1}{\sigma_i} \abs{ \v_i^{\top} \mat{Y}^{\top} \vec{a}  } \leq \frac{1}{\sigma_i} \norm{\v_i}_1 \linf{\mat{Y}^{\top} \vec{a}} \leq \frac{d}{\lambda} \linf{\mat{Y}^{\top} \vec{a}}.
\end{equation}
\end{proof}

\subsection{From Constrained to Unconstrained Lower Bounds}
\label{appendix:constrainwlog}

Here we provide a generic, black-box reduction from approximate Lipschitz
convex optimization over a unit ball to unconstrained approximate
Lipschitz convex optimization. The reduction leverages a natural extension
of functions on the unit ball to $\R^{d}$, provided and analyzed in the following
lemma.
\begin{lemma}
\label{lem:lift} 
Let $f:B_{2}^{d}\rightarrow\R$ be a convex,
$L$-Lipschitz function and let $g_{r,L}^{f}:\R^{d}\rightarrow\R$
and $h_{r,L}^{f}:\R^{d}\rightarrow\R$ be defined for all $\x\in\R^{d}$
and some $r\in(0,1)$ as
\[
g_{r,L}^{f}(\x)\defeq\begin{cases}
\max\{f(\x),h_{r,L}^{f}(\x)\} & \text{for }\norm \x_{2}<1\\
h_{r,L}^{f}(\x) & \text{for }\norm \x_{2}\geq1
\end{cases}\text{ and }h_{r,L}^{f}(\x)\defeq f(\bzero)+\frac{L}{1-r}\left((1+r)\norm \x_2-2r\right)\,.
\]
Then, $g_{r,L}^{f}(\x)$ is a convex, $L\left(\frac{1+r}{1-r}\right)$-Lipschitz
function with $g_{r,L}^{f}(\x)=f(\x)$ for all $\x$ with $\norm \x_{2}\leq r$.
\end{lemma}

\begin{proof}
For all $\x\in B_{2}^{d}$ note that 
\begin{equation}
h_{r,L}^{f}(\x)-f(\x)=f(\x)-f(\bzero)+\frac{L}{1-r}\left(\left(1+r\right)\norm \x_2-2r\right).\label{eq:bound}
\end{equation}
Now $|f(\x)-f(\bzero)|\leq L\norm \x_2$ since $f$ is $L$-Lipschitz. Correspondingly, when $\norm \x_2=1$ this implies
\begin{align}
h_{r,L}^{f}(\x)-f(\x) & \geq-L+\frac{L}{1-r}\left((1+r)-2r\right)=0\,.
\end{align}
Consequently, $h_{r,L}^{f}(\x)\geq f(\x)$ when $\norm \x_2=1$ and $g_{r,L}^{f}(\x)=h_{r,L}^{f}(\x)=\max\{f(\x),h(\x)\}$.
Further, since $h_{r,L}^{f}$ is $\left(\frac{1+r}{1-r}\right)$-Lipschitz
and convex and $f$ is $L$-Lipschitz and convex this implies that
$g_{r,L}^{f}$ is $\left(\frac{1+r}{1-r}\right)$-Lipschitz and convex
as well. Finally, if $\norm \x_{2}\leq r$ then again by (\ref{eq:bound})
and that $|f(\x)-f(\bzero)|\leq L\norm \x_2$ we have
\begin{align}
h_{r,L}^{f}(\x)-f(\x)  \leq L\norm \x_2+\frac{L}{1-r}\left((1+r)\norm \x_2-2r\right) =\frac{2L}{1-r}\left(\norm \x_2-r\right)\leq 0.
\end{align}
Therefore $g_{r,L}^{f}(\x)=f(\x)$ for all $\x$ with $\norm \x_{2}\leq r$.
\end{proof}

Leveraging this extension we obtain the following result.

\begin{lemma}\label{lem:constrained_reduction}
Suppose any $M$-memory-constrained randomized algorithm must make (with high probability in the worst case) at least
$\runtime_{\epsilon}$ queries to a first-order oracle for convex,
$L$-Lipschitz $f:B_{2}^{d}\rightarrow\R$ in order to compute an $\epsilon$-optimal point for
some $\epsilon<L$.\footnote{Note that $0$ queries are needed when $\epsilon\geq L$.}
Then any $M$-memory-constrained randomized algorithm must make at least $\runtime_{\epsilon}/2$ queries to a first-order
oracle (with high probability in the worst case) to compute an $\epsilon/2$-optimal
point for a $\bigo(L^{2}/\epsilon)$-Lipschitz, convex $g:\R^{d}\rightarrow\R$ even though the minimizer of $g$ is guaranteed to lie in $B_{2}^{d}$.
\end{lemma}

\begin{proof}
Let $\x \opt$ be a minimizer of $f$. By convexity, for all $\alpha\in[0,1]$
we have that 
\begin{align}
f(\alpha \x \opt)-f(\x \opt) & =f(\alpha \x+(1-\alpha)\bzero)-f(\x \opt)\\
 & \leq\alpha\left[f(\x \opt)-f(\x \opt)\right]+(1-\alpha)\left[f(\bzero)-f(\x \opt)\right]\\
 & \leq(1-\alpha)\norm{\bzero-\x \opt}_{2}\leq(1-\alpha)L\,.
\end{align}
Consequently, for all $\delta>0$ and $\alpha_{\delta}\defeq 1-\delta/L$
we have that $f(\alpha_{\delta} \x \opt)$ is $\delta$-optimal. 

Correspondingly, consider the function $g\defeq g_{\alpha_{\epsilon/2},L}^{f}$
as defined in Lemma \ref{lem:lift}. Note that $g$ is $L\left(\frac{1+\alpha_{\delta}}{1-\alpha_{\delta}}\right)=\bigo(L^{2}/\delta)$ Lipschitz. The minimizer of $g$ also lies in $B_2^d$ since $g$ monotonically increases for $\norm \x_{2}\geq1$. Suppose there is an algorithm to compute an
$\epsilon/2$-optimal minimizer of all $\bigo(L^{2}/\delta)$-Lipschitz, convex functions $g:\R^d \rightarrow \R$ whose minimizer lies in $B_2^d$, with high probability
using $\runtime$ queries to a first-order oracle for $g$. Note that
this algorithm can be implemented instead with $2\runtime$ queries
to $f$ by simply querying $f(\bzero)$ along with any query, and then using the defining of $g$ from Lemma \ref{lem:lift} (note that it is also possible to just compute and store $f(\bzero)$ once in the beginning and hence only require $\runtime+1$ queries, but this would require us to also analyze the number of bits of precision to which $f(\bzero)$ should be stored, which this proof avoids for simplicity). If $\x_{\mathrm{out}}$ is
the output of the algorithm then
\begin{align}
\frac{\epsilon}{2} & \geq g(\x_{\mathrm{out}})-\inf_{x}g(\x)\geq g(\x_{\mathrm{out}})-g(\alpha_{\epsilon/2}\x \opt)\geq g(\x_{\mathrm{out}})-f(\x \opt)-\frac{\epsilon}{2}\,.
\end{align}
Thus $g(\x_{\mathrm{out}})\leq f(\x \opt)+\epsilon$. Let $\tilde{\x}_{\mathrm{out}}=\min\{1,\norm{\x_{\mathrm{out}}}_{2}^{-1}\}\x_{\mathrm{out}}$. As $g$ monotonically increases for $\norm \x_{2}\geq1$, we have that $g(\x_{\mathrm{out}})\geq g(\tilde{\x}_{\mathrm{out}})=f(\tilde{\x}_{\mathrm{out}})$
and thus with $2\runtime$ queries to $f$ the algorithm can compute
an $\epsilon$-approximate minimizer of $f$. Consequently, $\runtime\geq\runtime_{\epsilon}/2$
as desired and the result follows.
\end{proof}

In this section we focused on the unit $\ell_2$ norm ball, however the analysis generalizes to arbitrary norms.
\end{document}